\documentclass[journal]{IEEEtai}
\usepackage{soul}
\usepackage{url}
\usepackage[hidelinks]{hyperref}
\usepackage[T1]{fontenc}
\usepackage[utf8]{inputenc}
\usepackage{graphicx}
\usepackage{amsmath}
\usepackage{amsfonts}
\usepackage{amsthm}
\newtheorem{proposition}{Proposition}

\usepackage{amssymb}  
\usepackage{booktabs}
\usepackage{algorithm}
\usepackage{algorithmic}
\usepackage{tcolorbox}
\usepackage{tabularx}
\usepackage{multirow}
\usepackage[table]{xcolor}
\begin{document}

\title{Beyond the Node: Clade-level Selection for MCTS in Automatic Heuristic Design}

\author{Kezhao~Lai, 
        Yutao~Lai, 
        and~Hai-Lin~Liu
\thanks{Kezhao Lai, Yutao Lai, and Hai-Lin Liu are with the School of Mathematics and Statistics, Guangdong University of Technology, China. (Corresponding author: Hai-Lin Liu, e-mail: hlliu@gdut.edu.cn).}%
}



\maketitle

\begin{abstract}
While Monte Carlo Tree Search (MCTS) shows promise in Large Language Model (LLM) based Automatic Heuristic Design (AHD), it suffers from a critical over-exploitation tendency under the limited computational budgets required for heuristic evaluation. To address this limitation, we propose Clade-AHD, a framework that replaces node-level point estimates with clade-level Bayesian beliefs. By aggregating descendant evaluations into Beta distributions and performing Thompson Sampling over these beliefs, Clade-AHD explicitly models uncertainty to guide exploration, enabling more reliable decision-making under sparse and noisy evaluations. Extensive experiments on complex combinatorial optimization problems demonstrate that Clade-AHD consistently outperforms state-of-the-art methods while significantly reducing computational cost. The source code is publicly available at: \url{https://github.com/Mriya0306/Clade-AHD}.
\end{abstract}





\begin{IEEEImpStatement}
Combinatorial optimization algorithms are the cornerstone of modern logistics, supply chains, and chip manufacturing, yet traditional design relies heavily on expensive human expertise. While Large Language Models (LLMs) show immense potential in automating algorithm design, their unguided trial-and-error processes often incur prohibitive computational costs and energy consumption. The efficient heuristic generation framework proposed in this paper breaks this bottleneck, enabling the automated design of superior optimization algorithms under minimal computational budgets. From an economic and environmental perspective, this breakthrough significantly lowers the barrier for enterprises to customize exclusive optimization solutions. Logistics and manufacturing industries can rapidly deploy more efficient scheduling systems, substantially reducing both operational costs and transportation-related carbon emissions. Furthermore, by drastically minimizing futile computational waste during AI reasoning, this technology promotes reduced carbon footprints and lays a solid foundation for the large-scale industrial deployment of green, sustainable artificial intelligence.
\end{IEEEImpStatement}


\section{Introduction}
Combinatorial Optimization (CO) problems are a cornerstone of modern science and engineering, with critical applications in fields such as logistics and route planning~\cite{route}, supply chain optimization~\cite{supply_ch}, chip design~\cite{chips1,chips2}, and drug discovery~\cite{drug1,drug2}. The majority of these problems are NP-hard~\cite{Garey1979}, making the search for optimal solutions computationally prohibitive. Consequently, both academia and industry have long focused on designing efficient heuristic algorithms to find high-quality approximate solutions. Nevertheless, traditional heuristic design remains a critical bottleneck, constrained by its reliance on manual expertise and iterative trial-and-error. This dependency results in prohibitive costs.

The emergence of Large Language Models (LLMs) provides a transformative opportunity to overcome these constraints. By functioning as creativity engines, these models enable the rapid exploration of vast algorithmic spaces. Recent advancements, such as LLaMoCo~\cite{ma2026llamoco}, demonstrate that instruction tuning tailored to optimization tasks can significantly enhance the professional quality of generated code. This progress suggests that LLMs can accurately capture and implement complex mathematical logic within diverse programming contexts.

Recently, LLMs have catalyzed a paradigm shift in addressing this enduring challenge, positioning them as creativity operators\cite{romera2024mathematical, Mankowitz2023} within evolutionary frameworks. This concept builds on a rich history of ideas in Automated Machine Learning (AutoML)~\cite{Hutter2019,Feurer2015} and Genetic Programming (GP)~\cite{Koza1992}, further expanding the design space of AHD. Pioneering works such as Funsearch~\cite{romera2024mathematical} and AlphaDev~\cite{Mankowitz2023} have demonstrated the profound potential of this synergy by discovering algorithms that surpass human-crafted baselines. EOH~\cite{liu2024evolution} pioneered the co-evolution of natural language thoughts and code to mimic human iterative design. Building on these foundations, LLaMEA utilizes an evolutionary loop to generate tailored metaheuristics that adapt to specific problem instances~\cite{van2024llamea}. Simultaneously, ReEvo introduces a reflective evolution mechanism where the LLM operates as a hyper-heuristic by learning from historical performance to iteratively refine search logic~\cite{ye2024reevo}. These frameworks effectively automate the design process but share a critical structural limitation inherent to evolutionary populations. By maintaining a strictly bounded pool of candidates and prioritizing immediate fitness, these methods perpetually discard intermediate variations. This greedy culling mechanism systematically destroys the historical topology of the search process. Consequently, the algorithm becomes highly susceptible to local optima because it discards transiently mediocre heuristics that actually serve as essential critical intermediate nodes toward the global optimum.

Monte Carlo Tree Search naturally resolves this structural limitation by providing a persistent topology for global exploration. MCTS-AHD integrates this architecture to retain the complete evolutionary history~\cite{Zheng2025}. However, deploying MCTS for code evolution reveals a fundamental algorithmic mismatch. The exorbitant computational costs associated with evaluating generated heuristics force the search tree into a regime of extreme data sparsity. Traditional tree policies rely on asymptotic convergence and frequentist point estimates of rewards. Under strictly bounded simulations, these point estimates become highly unreliable. Furthermore, relying exclusively on scalar reward values neglects the continuous variance and epistemic uncertainty inherent in partial evolutionary branches. This condition necessitates a transition from frequentist node evaluation to a comprehensive probabilistic reward model.

 To address this challenge, we introduce \textbf{Clade-AHD}, a framework that systematically reformulates automatic heuristic design through a hierarchical Bayesian reward model. Rather than depending on isolated point estimates of heuristic fitness, Clade-AHD constructs a probability distribution of rewards for an entire evolutionary lineage, as illustrated in Figure~\ref{fig:mechanism_comparison}. By aggregating genealogical evidence into a unified probabilistic reward space, the framework explicitly models exploration uncertainty and mitigates estimation variance. This approach avoids the pitfalls of scalar estimations and successfully identifies highly promising evolutionary clades even when observations remain critically sparse. Our specific contributions are:  

(1) \textbf{Hierarchical Bayesian Abstraction}: We redefine the search space from node-level point estimates to Clade-level Bayesian beliefs, modeling heuristic potential as Clade-level Beta distributions. To robustly infer these distributions from sparse observations, we propose a \textbf{Clade-Aware Belief Update} mechanism that aggregates genealogical evidence bottom-up while strictly enforcing causal relevance via depth attenuated credit assignment, thereby mitigating the impact of noisy evaluations and preventing premature commitment to local optima.

(2) \textbf{Uncertainty Guided Exploration Policy}: We introduce a robust selection strategy governed by \textbf{Clade-Level Thompson Sampling}. By integrating prior stabilization with a budget-aware dynamic annealing, this policy effectively navigates the exploration-exploitation trade-off. It progressively shifts search focus from high-uncertainty clades to high-potential clades as the evaluation budget is exhausted.

\section{Related Work}

\subsection{LLM-based AHD}
\subsubsection{Mathematical Paradigm of AHD}

In addressing complex combinatorial optimization (CO) tasks $\mathcal{T}$, such as the Capacitated Vehicle Routing Problem (CVRP), the core objective of Automated Heuristic Design (AHD) is to precisely extract the globally optimal heuristic policy $\phi^{\star}$ from an expansive algorithmic space $\Phi$. This optimization objective can be formally defined as:
\begin{equation}
    \phi^{\star} = \arg\max_{\phi \in \Phi} \mathcal{E}(\phi).
\end{equation}
Formally, a candidate policy $\phi$ operates as a deterministic decision function, denoted by $\phi: X_{\mathcal{T}} \rightarrow Y_{\mathcal{T}}$. Consider the CVRP as a concrete example. The input state $X_{\mathcal{T}}$ comprises spatial coordinates and specific customer requirements. This function evaluates these variables to generate a valid delivery sequence $Y_{\mathcal{T}}$. Ultimately, this sequence must satisfy the physical load constraints of the vehicles.

Evaluating a policy $\phi$ requires a rigorously defined cost function $\mathcal{C}$, such as the total route length. Exact computation across the entire state space is intractable. Therefore, AHD frameworks rely on empirical sampling. By utilizing a specialized validation set $\mathcal{D}_{eval}$, we approximate the overall performance metric $\mathcal{E}(\phi)$, which is formulated as the negative expectation of the solution costs across all evaluated instances $x \in \mathcal{D}_{eval}$:
\begin{equation}
    \mathcal{E}(\phi) = \mathbb{E}_{x \in \mathcal{D}_{eval}}[-\mathcal{C}(\phi(x))].
\end{equation}

The complete policy space $\Phi$ is notoriously unstructured, rendering direct search methods largely intractable. To circumvent this bottleneck, modern AHD practices typically restrict the search scope within predefined algorithmic templates. This strategic confinement significantly streamlines the design process. Rather than synthesizing an entire solver from scratch, computational effort is exclusively directed toward optimizing the core scoring module $\phi$ that guides sequential state transitions.

\subsubsection{LLM-Driven Policy Evolution}

Recently, the exceptional code-generation capabilities of Large Language Models have been extensively integrated into the AHD domain to accelerate the evolution of $\phi^{\star}$ within the aforementioned frameworks. These LLM-driven AHD methods, exemplified by frameworks such as \textbf{ReEvo}~\cite{ye2024reevo} and its successor \textbf{HSEvo}~\cite{dat2025hsevo}, deeply fuse the principles of Evolutionary Computation (EC). Their fundamental mechanism involves maintaining an elite policy population $\{\phi_{1}, \dots, \phi_{K}\}$ of fixed size $K$. In each iteration, the system extracts high-performing heuristic code from the pool as context-aware prompts, inducing the LLM to recombine and generate novel algorithms $\phi$ through logic analogous to "crossover" or "mutation" in EC.

However, such population-based evolutionary frameworks expose significant limitations in exploration. The primary bottleneck stems from a rigid and myopic elitist selection strategy. Every new policy $\phi$ synthesized by the LLM undergoes immediate comparative evaluation. To survive, its performance score $\mathcal{E}(\phi)$ must strictly surpass the weakest individual in the current pool. We define this exact survival condition as $\mathcal{E}(\phi) > \min_{j \in \{1,\dots, K\}} \mathcal{E}(\phi_{j})$. The system automatically discards any candidate failing to satisfy this stringent requirement.

This greedy selection method only focuses on immediate results. As a result, the algorithm loses the flexibility to explore during early search stages. This short-term focus makes the process likely to get stuck early. The system simply rejects structural changes that cause a temporary drop in performance. However, these temporary drops are often necessary for finding the global optimum. Ultimately, this strict filtering process greatly limits the exploration of the heuristic space.

\subsection{LLMs as optimizers}
\subsubsection{Algorithm Generation Paradigm}

Recently, researchers have increasingly recognized that directly prompting generalist Large Language Models is insufficient for complex optimization tasks. In contrast, several studies explore another possibility: prompting LLMs directly for pieces of executable optimization programs. For example, Liu et al. proposed a general development platform, \textit{LLM4AD}~\cite{liu2024evolution}, to help practitioners in the related domain to learn, develop, and evaluate. 

However, unlike tasks such as language understanding, optimization problems are difficult for humans to solve without efficient algorithms, fundamentally challenging the LLMs' reasoning and generalization abilities. \textit{OptiMUS}, on the other hand, integrates hints about the target optimizer into the prompts to create and test new optimizers for numerical optimization problems. Because pre-trained models inherently lack domain specific expertise, fine tuning them to improve their ability to solve complex problems has become a critical yet underexplored research frontier~\cite{ma2026llamoco}.

\subsubsection{Iterative Black Box Optimization}

In parallel, recent methodologies have repositioned LLMs as black-box optimizers, prompting them to execute evolutionary operators—such as mutation, crossover, and selection—directly on candidate solutions. A prominent example is \textit{EvoPrompting}, which demonstrates remarkable efficiency in discovering state-of-the-art neural architectures. These frameworks typically maintain a population of solutions, using LLMs to iteratively refine them based on objective feedback and historical scores until a predefined stopping criterion is met. 

Extensive research has explored the use of LLMs for numerical optimization and automated optimizer discovery, with most approaches relying on prompt engineering and few-shot learning~\cite{Hutter2019, Feurer2015}. Forniés-Tabuenca, Latorre, and Ramos introduced a reflection-based self-refinement mechanism into the iterative prompting process, allowing the model to produce better optimization programs over time. This line of work is promising, but its efficiency is often limited by the large number of iterations required. Huang, Wang, and colleagues also released a comprehensive set of prompting strategies and highlighted the importance of the repair process in LLM-generated programs.

\subsubsection{Fine-Tuning in Evolutionary Design}

An emerging paradigm integrates LLMs into the field of Evolutionary Computation (EC) for automated algorithm design, employing these models to search for and evolve executable optimization scripts~\cite{van2024llamea}. By evolving the programs of constructive methods, the finally obtained optimization program could beat certain human-crafted heuristics. While prior studies have explored instruction-tuning for operations research problems, low-cost LLM-aided multi-objective optimization, and the use of LLMs as surrogate models in Evolutionary Algorithms (EAs), these efforts have been limited in scope. 

They have focused on discrete domains or narrow component-level assistance, rather than the end-to-end generation of complete optimizers for the continuous domain. This highlights the contribution of \textit{LLaMoCo}~\cite{ma2026llamoco}. It is a pioneering framework for tuning LLMs for this task and addresses both data preparation and training enhancement. To streamline algorithm design, Zhong et al. proposed a one-shot prompting strategy based on the CRISPE framework~\cite{zhong2024leveraging}. Their method successfully synthesizes novel metaheuristics without requiring iterative dialogue. Unlike existing evolutionary optimizers, the generated algorithm demonstrates unique search behaviors. Ultimately, this work highlights the immense potential of LLMs as independent algorithm designers.

Previous works often rely strictly on standard LLM textual mutations. \textit{LLaMEA}~\cite{van2024llamea} identifies this reliance as a key bottleneck. To address this limitation, LLaMEA integrates an evolutionary loop with LLM-based mutation, enabling the generation of adaptive algorithm scripts without explicit length constraints. Consequently, their method can generate adaptive algorithm scripts without length constraints. This direction builds naturally upon the foundational successes of \textit{FunSearch}~\cite{romera2024mathematical} and \textit{EoH}~\cite{liu2024evolution}. Overall, these studies underscore the growing interest in leveraging LLMs for iterative heuristic program discovery, as reflected in a rapidly expanding body of recent survey and methodological work~\cite{bartz2014hyper}.

\section{Preliminary}

This section provides a comprehensive theoretical grounding for the Clade-AHD framework, situating our contributions within the broader context of evolutionary computation, biological clade selection, and Bayesian reinforcement learning.

\subsection{From Node-Centric to Clade-Centric Search}
Traditional MCTS algorithms, such as UCT \cite{kocsis2006bandit}, operate under the assumption that nodes function as independent bandits. This framework is theoretically sound in domains that permit abundant sampling budgets like board games~\cite{Silver2016}. Yet, this independence assumption disintegrates in the context of Automatic Heuristic Design (AHD). In AHD, a heuristic is not an isolated entity but a refined continuation of its ancestral logic.

We adopt the concept of \textit{Clade Selection} from evolutionary biology to formalize this dependency. As argued by Okasha \cite{okasha2003concept} in the philosophy of biology, strictly individual-level selection fails to account for traits that manifest only at the level of a lineage. Analogously, in our code evolution tree, a specific mutation might yield a statistically insignificant reward in isolation (due to evaluation noise) yet possess the structural potential to spawn a high-performing sub-tree. Clade-AHD thus shifts the unit of selection from the individual code node to the \textit{clade}. We formally define a clade as an ancestor node alongside all of its lineal descendants. This expanded scope naturally captures the \textit{meta-productivity} of an entire evolutionary branch, providing a far more reliable evaluation metric than the transient fitness of a single, isolated code snippet.

\subsection{Bayesian Inference in Sparse Evaluation Environments}
The core inefficiency of standard AHD methods lies in their reliance on frequentist estimation. The Law of Large Numbers underpins the asymptotic convergence of mean-based estimators like those in UCT. However, this statistical law is effectively suspended in LLM-driven optimization due to the prohibitive cost of heuristic evaluation. When the node visit count remains critically small (often below a strict threshold $N<5$), scalar point estimates become overwhelmed by variance. This severe limitation reduces the structured tree search to an unguided stochastic walk.

To address this, Clade-AHD leverages Bayesian inference to model explicit epistemic reward uncertainty. Building on the seminal work of Tesauro et al. \cite{tesauro2012bayesian}, we model scalar evaluations as probabilistic reward distributions. Unlike deterministic Upper Confidence Bounds, which can be overly optimistic or pessimistic in sparse regimes, our use of Thompson Sampling \cite{bai2013bayesian} enables \textit{probability matching} over these reward models. This ensures that the exploration effort is proportional to the probability that a clade delivers the optimal expected reward, a property that has been theoretically proven to minimize regret in bandit settings with delayed or sparse feedback \cite{chapelle2011empirical}.


\subsection{LLM-driven Heuristic Design}

Automating algorithm design with computational intelligence is a long-standing research area, with origins in AutoML~\cite{Hutter2019, Feurer2015,song2024efficient}, and GP~\cite{Koza1992,liao2024learning}. The advent of LLMs, with their remarkable capabilities in code understanding, generation, and reasoning~\cite{Bubeck2023}, has injected unprecedented vitality into this field. Currently, LLM-driven AHD is dominated by two technical paradigms, each making a different trade-off between exploration and exploitation.

 Population-based methods currently dominate the field of LLM-driven AHD. These methods maintain a pool of high-quality heuristics and iteratively optimize them using evolutionary principles. \textbf{Funsearch}~\cite{romera2024mathematical} established this paradigm by employing an LLM as a mutator to discover novel mathematical solutions. \textbf{EOH}~\cite{liu2024evolution} later advanced this framework with sophisticated genetic operators and maintenance strategies, showing strong performance on combinatorial tasks. However, EOH suffers from a critical structural flaw: its reliance on greedy, fitness-based selection. To strictly enforce a fixed population limit, heuristics that are temporarily underperforming but possess high evolutionary potential are prematurely culled. While this strategy accelerates convergence, it sacrifices population diversity, trapping the search in local optima. To overcome the local optima problem inherent in population-based methods, researchers have explored strategies with a more global perspective. Inspired by the success of \textbf{AlphaGo}~\cite{Silver2016} in navigating vast game spaces, this method introduces MCTS to the heuristic evolution process. MCTS-AHD~\cite{Zheng2025} organizes all generated heuristics in a continuously growing search tree, using the UCT algorithm to balance exploitation and exploration. By retaining all historical information, this paradigm successfully mitigates the short-sightedness of population-based methods and significantly improves the quality of the discovered heuristics. Nevertheless, MCTS-AHD~\cite{Zheng2025} is structurally prone to over-exploitation under limited evaluation budgets. Its reliance on point estimates without sufficient sampling makes the search policy susceptible to stochastic noise. This frequently leads to premature convergence, where the algorithm wastes resources optimizing local optima rather than maintaining effective global exploration.

\subsection{Limitations} 

\paragraph{Limitation of Standard MCTS in AHD Scenario}
MCTS has achieved remarkable success in sequential decision-making problems, particularly in game playing as exemplified by AlphaGo~\cite{Silver2016}. The algorithm typically operates through four iterative phases: Selection, Expansion, Simulation, and Backpropagation. During the Selection phase, standard MCTS relies on the Upper Confidence Bound for Trees (UCT)~\cite{kocsis2006bandit} to balance exploration and exploitation. The UCT policy selects a child node $a$ of node $v$ according to:
\begin{equation}
\text{UCT}(a,v) = \textbf{}{Q}(a,v) + C \sqrt{\frac{\ln (N(v)+1)}{N(a)}},
\end{equation}
where $\textbf{}{Q}(a,v)$ represents the newly expended leaf nodes for their quality value, $N(v)$ and $N(a)$ denote the visit counts, and $C$ is an exploration constant.


While UCT theoretically guarantees asymptotic convergence by treating nodes as independent bandits, this fundamental premise disintegrates in LLM-based AHD~\cite{tesauro2012bayesian}. The prohibitive computational cost of heuristic generation and evaluation forces the search tree into a regime of extreme visit sparsity (commonly $N \le 3$). Under such stringent constraints, the reliance on Hoeffding-based confidence bounds causes a structural failure of the search policy. We formalize this fundamental incompatibility through three propositions.

\paragraph{Vacuousness of the Hoeffding Bound.}
The standard UCT exploration bonus, proportional to $\sqrt{\ln N(v)/N(a)}$, is derived from Hoeffding's inequality. For rewards bounded in $[0,1]$, it assumes that empirical means reliably track true expectations. However, this bound entirely neglects empirical variance, unlike tighter alternatives that remain sensitive to it.

\begin{proposition}\label{prop:variance}
Let sibling nodes $a$ and $b$ have equal sparse visit counts $N(a)=N(b)=k \le 3$ and identical empirical means $\bar{Q}(a,v)=\bar{Q}(b,v)$. Then $\mathrm{UCT}(a,v)=\mathrm{UCT}(b,v)$ is unconditionally true, completely disregarding any disparity in their empirical variances.
\end{proposition}
\begin{proof}
By directly substituting the assumed conditions into the UCT formulation, the exploration terms are identical, yielding $\mathrm{UCT}(a,v) - \mathrm{UCT}(b,v) = \bar{Q}(a,v) - \bar{Q}(b,v) = 0$. This insensitivity to higher moments is particularly consequential in LLM-based AHD, where variance proxies evolutionary potential. A stable but mediocre mutation with normalized returns $[0.4, 0.4, 0.4]$ receives the exact same UCT priority as a high-risk and high-reward mutation with returns $[0.0, 1.0, 0.2]$ (mean $0.4$). UCT inherently fails to distinguish between barren stability and high-variance evolutionary potential.
\end{proof}

\begin{proposition}\label{prop:collapse}
For a node with $n$ evaluations, achieving a statistical confidence level $1-\delta$ under Hoeffding's inequality requires an interval half-width of $w(n,\delta)=\sqrt{\ln(2/\delta)/(2n)}$.  Under the extreme sparsity of AHD (e.g., $n \le 3$), establishing even a relaxed $95\%$ confidence ($\delta=0.05$) yields:
\begin{align*}
w(1,0.05)&\approx 1.358,\\
w(2,0.05)&\approx 0.960,\\
w(3,0.05)&\approx 0.784.
\end{align*}
Because the normalized metric space of AHD is strictly bounded within $[0,1]$, an interval half-width $w \ge 0.5$ causes the Upper Confidence Bound (UCB) to exceed the maximum theoretically possible reward ($1.0$). For $n \le 3$, the Hoeffding guarantee collapses into a mathematically vacuous statement, failing to provide any informative statistical upper bound.
\end{proposition}

\begin{proposition}\label{prop:ctuning}
The structural flaws of UCT under extreme sparsity cannot be circumvented by merely tuning the exploration constant $C$.
(i) If $N(i)=N(j)$, the UCT difference is independent of $C$ (Proposition~\ref{prop:variance}).
(ii) If $N(i)<N(j) \le 3$, as $C \to \infty$, the policy degenerates: $\lim_{C\to\infty}\arg\max_k \mathrm{UCT}(k,v) = \arg\min_k N(k)$.
\end{proposition}
\begin{proof}
Follows directly from the cancellation of the exploration terms. (ii) As $C$ dominates, the empirical mean $\bar{Q}$ becomes negligible. The UCT ratio between nodes $i$ and $j$ converges to $\sqrt{N(j)/N(i)} > 1$. Hence inflating $C$ does not induce variance-aware sampling. Instead, it forces the search policy into a deterministic and uninformed round-robin over the least-visited nodes, indiscriminately wasting the evaluation budget on structurally broken heuristic branches.
\end{proof}

Empirical evaluations (see Table~\ref{tab:ablation_exploration}) demonstrate that simply tuning the exploration constant $C$ is ineffective in AHD; indiscriminately increasing $C$ merely triggers a non-productive random walk over low-visit nodes. This is because $C \sqrt{\ln N/n_i}$ becomes mathematically vacuous when $n_i \le 3$, failing to distinguish between stable mediocrity and high-potential exploratory mutations~\cite{tesauro2012bayesian,misaki2025wider}.

\paragraph{Limitations of Node-Centric Bayesian MCTS} 
To address sample inefficiency, prior works have integrated Bayesian inference into the search framework. Moving beyond scalar estimates, methods such as Bayes-UCT~\cite{tesauro2012bayesian} and Thompson Sampling MCTS~\cite{bai2013bayesian} model potential as probability distributions. This allows the policy to account for epistemic uncertainty, prioritizing actions with high information gain. However, these approaches remain strictly \textbf{node-centric}. They operate on the assumption that a state's value is intrinsic and independent of its siblings, which is a flawed assumption in heuristic design, where an intermediate code fragment, acting as a stepping stone, may perform poorly in isolation yet enable superior descendants.

\section{The Proposed CLADE-AHD}
\begin{figure*}[t] 
    \centering
    \includegraphics[width=0.95\textwidth]{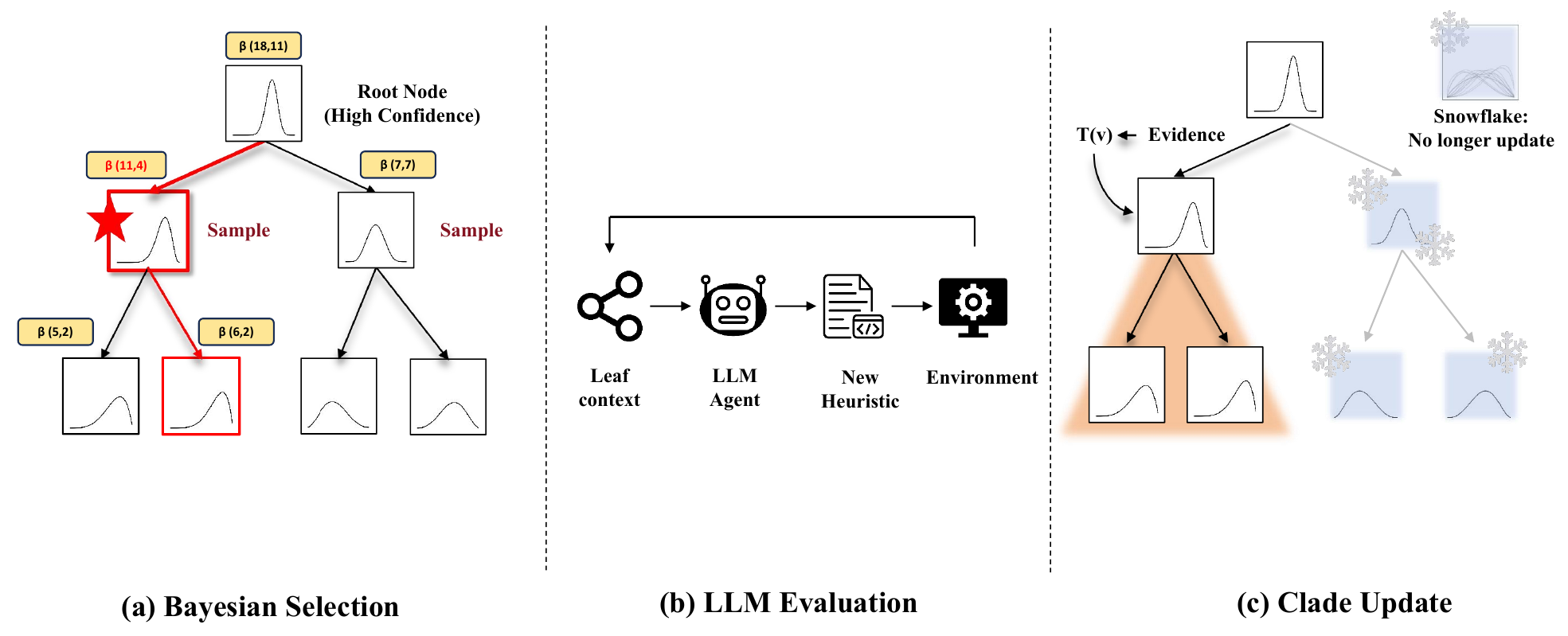} 
    \caption{
        The schematic overview of the Clade-AHD framework. 
        (a) \textbf{Bayesian Selection}: The policy navigates the tree via Clade-level Thompson Sampling, selecting nodes based on belief distributions. 
        (b) \textbf{LLM Evaluation}: The LLM generates and evaluates new heuristics based on the leaf context. 
        (c) \textbf{Clade Update}: Evidence propagates bottom-up to refine ancestral beliefs (sharper curves), while the \textit{Dynamic Clade Freezing} mechanism (snowflake icons) prunes suboptimal branches.
    }
    \label{fig:framework_overview}
\end{figure*}


We introduce Clade-AHD, a framework that shifts the evaluation paradigm from node-level scalar estimation to Clade-level distributional belief. By modeling the potential of each search branch as a Beta distribution and propagating evidence bottom-up, we enable the system to identify transitional nodes that may appear modest in isolation but act as stepping stones to high-value evolutionary descendants. The framework consists of five integrated components:

\subsection{Bayesian Abstraction of Search Space}

Driven by the inherent unreliability of point estimates in sparse search spaces, we suggest that heuristic evolution in MCTS is better understood as a Clade-level potential inference rather than a node-level point estimation. An individual node's value is transient and noisy, whereas the clade represents a robust, evolving belief state.

\paragraph{Definition 1: Clade}
We formally introduce the concept of a \textbf{clade} to capture the structural correlation of heuristics. Borrowing from biological taxonomy, a clade represents an ancestor and all its evolutionary descendants. Mathematically, for a node $v$ in the current search tree $\mathbb{T}$, we define its clade $\mathcal{T}(v)$ as the set containing $v$ all its progeny: \begin{equation}
    \mathcal{T}(v) = \{v\} \cup \{u \in \mathbb{T} \mid v \text{ is an ancestor of } u\}.
\end{equation} This set definition is crucial because it delineates the scope of evidence aggregation. In our framework, a node $v$ is not judged solely by its immediate fitness, but serves as the belief holder for the entire set $\mathcal{T}(v)$. This allows us to capture the value of high-potential heuristics, which are nodes that may perform poorly in isolation but possess the structural potential to generate superior descendants.
\paragraph{Definition 2: Beta Distribution}
To quantify the potential of this structural unit, we map the fitness scores of each clade to a probabilistic belief space. While existing literature often models value distributions using Gaussian mean-variance approaches, such assumptions may become unreliable under extremely sparse observations. Gaussian empirical variance collapses under extreme sparsity (e.g., undefined for $N=1$) and its unbounded support $(-\infty, +\infty)$ permits physically impossible heuristic evaluations. In contrast, by normalizing continuous objectives to $[0,1]$, our framework leverages the Beta distribution as a conjugate prior. This ensures that the parameters $\alpha_v$ and $\beta_v$ carry precise physical meaning as fractional pseudo-counts of positive and negative evolutionary evidence, explicitly and robustly modeling epistemic uncertainty even with zero or one evaluation. We model the node's potential as a random variable following a Beta distribution $\mathcal{B}(\alpha_v, \beta_v)$\cite{beta1}, which serves as the conjugate prior for the Bernoulli likelihood of heuristic evaluations. The probability density function is defined as:
\begin{equation}
    f(\theta; \alpha_v, \beta_v) = \frac{\theta^{\alpha_v - 1} (1 - \theta)^{\beta_v - 1}}{B(\alpha_v, \beta_v)},
\end{equation}
where $B(\cdot)$ represents the Beta function serving as a normalization constant, and we do not treat \textit{v} merely as the visitation state of an individual node, but as the current statistical representative of its entire evolutionary clade $\mathcal{T}(v)$. This Bayesian abstraction allows us to represent the complete distribution of a heuristic's performance, capturing both its expected quality and the associated uncertainty. Formally, the belief state of a node $v$ is parameterized by the counts of successes $\alpha_v$ and failures $\beta_v$, with the expectation and variance given by:\begin{equation}\mathbb{E}[\theta_v] = \frac{\alpha_v}{\alpha_v + \beta_v}, \qquad\text{Var}[\theta_v] = \frac{\alpha_v \beta_v}{(\alpha_v+\beta_v)^2(\alpha_v+\beta_v+1)}.\end{equation}

Strictly speaking, the Beta distribution is conjugate to the binomial likelihood and assumes discrete binary trials. A standard theoretically rigorous approach to handle continuous rewards $s \in [0,1]$ in bandit settings is to perform Bernoulli sampling: treating $s$ as the probability of observing a discrete success. However, we deliberately eschew this approach in favor of fractional pseudo-counts. In the context of LLM-based AHD, heuristic evaluation is computationally exorbitant, leading to extreme evaluation sparsity (often $N \le 3$). Under such stringent constraints, introducing an additional layer of Bernoulli sampling noise would cause catastrophic variance in belief estimation. Therefore, directly treating the normalized continuous scores as fractional updates acts as a deterministic expectation update. This serves as a vital variance-reduction modeling choice, ensuring stable probabilistic beliefs without exceeding the physical boundaries of the performance metrics.

\subsection{Clade Belief Updates}
Our belief update mechanism is inspired by the \textbf{Clade-Metaproductivity (CMP)} framework proposed in the Huxley-Gödel Machine\cite{wang2025huxley}. HGM pioneered the use of clade-based statistics in tree search, demonstrating that aggregating descendant evaluations yields more robust estimates than isolated node values. 

However, standard CMP treats all descendants equally, ignoring the temporal distance between an ancestor and its distant offspring. This is problematic in heuristic design, where accumulative semantic drift inevitably weakens the correlation between an ancestor and its deep descendants. As the evolutionary path lengthens, the performance of a node becomes increasingly attributable to recent mutations rather than the ancestor's initial logic. Consequently, uniform aggregation is prone to credit misassignment, permitting uncorrelated variations from distant offspring to induce erroneous value estimates for their ancestors. 

To address this issue, we propose a \textbf{Clade-aware Belief Update} mechanism that aggregates performance evidence from the entire clade $\mathcal{T}(v)$ rooted at $v$. Unlike traditional MCTS, which typically relies on node-local statistics, our approach establishes a \textbf{Bottom-up Evidence Propagation}. Similar to how reliable value estimation requires grounding in verifiable outcomes, we ground the belief of an ancestor node in the accumulated evidence of its entire clade. This propagates the discovery of high-quality descendants back to the root, effectively improving the evaluation signal for the entire branch\cite{updategood}. Formally, the clade-level belief parameters $\alpha_{\mathcal{T}}(v)$ and $\beta_{\mathcal{T}}(v)$ are derived by aggregating evidence from all descendants $u \in \mathcal{T}(v)$, weighted by a decay factor $\lambda \in (0, 1]$ to account for temporal credit assignment:
\begin{align}\alpha_{\mathcal{T}}(v) &= 1 + \sum_{u \in \mathcal{T}(v)} \lambda^{\text{dist}(v,u)} (\alpha_u - 1),\end{align}
\begin{align}\beta_{\mathcal{T}}(v)  &= 1 + \sum_{u \in \mathcal{T}(v)} \lambda^{\text{dist}(v,u)} (\beta_u  - 1).\end{align}
where $\text{dist}(v,u)$ represents the depth difference between node $v$ and $u$. Unlike the temporal discount factor $\gamma$ in standard Reinforcement Learning (e.g., MuZero~\cite{schrittwieser2020mastering}), our decay factor $\lambda$ addresses the phenomenon of \textit{semantic drift}. As LLMs mutate code across generations, a distant descendant may become conceptually decoupled from its ancestor's logic. By applying $\lambda$ directly to historical evidence, we prevent uncorrelated variations from deep branches from disproportionately corrupting ancestral value estimation, thereby ensuring that the backpropagation remains evolutionary coherent~\cite{tesauro2012bayesian,misaki2025wider}. This depth-attenuated aggregation ensures that the search is driven by the branch's long-term potential, while $\lambda$ prevents deep, specific variations from overly dominating the search direction of ancestor nodes. Crucially, as more descendants are explored, the accumulated counts ($\alpha_{\mathcal{T}}, \beta_{\mathcal{T}}$) increase, naturally reducing the variance of the belief distribution and yielding a progressively more confident estimation of the clade's true value. 

\subsection{Clade-level Thompson Sampling}
We govern the search trajectory via \textit{Clade-level Thompson Sampling}. Unlike rigid deterministic policies, this strategy naturally balances exploration and exploitation by sampling from the posterior belief of each clade. First, to mitigate estimator variance in sparsely visited regions, we regularize the aggregated clade evidence using expectation-based pseudo-counts \cite{landa}. Let $\alpha_{\mathcal{T}}(v)$ and $\beta_{\mathcal{T}}(v)$ be the aggregated evidence for the clade rooted at $v$. The baseline expected potential is defined as $\mathbb{E}[\theta_{\mathcal{T}}(v)] = \alpha_{\mathcal{T}}(v) / (\alpha_{\mathcal{T}}(v) + \beta_{\mathcal{T}}(v))$. This acts as a \textbf{Prior Stabilization} mechanism to bound initial variance:
\begin{align}
\tilde{\alpha}_{\mathcal{T}}(v) &= \alpha_{\mathcal{T}}(v) + n_{\text{pseudo}} \cdot \mathbb{E}[\theta_{\mathcal{T}}(v)], \\
\tilde{\beta}_{\mathcal{T}}(v) &= \beta_{\mathcal{T}}(v) + n_{\text{pseudo}} \cdot (1 - \mathbb{E}[\theta_{\mathcal{T}}(v)]).
\end{align}

However, standard Thompson Sampling implicitly assumes an infinite time horizon. To reconcile this with the strict computational constraints of LLM-based design, we introduce a \textbf{Budget-Aware Annealing} mechanism. We modulate the precision of our beliefs via a temperature parameter $\tau(p)$ that evolves with search progress $p = t/T_{\max}$:
$$ \tau(p) = \left(\frac{1}{1-p}\right)^{\omega}, $$
where $\omega$ controls the annealing rate. The final selection policy operates via a Stochastic Argmax over independent belief realizations. For each candidate child $v$, a scalar sample $\hat{\theta}_v$ is independently drawn using the scaled posterior:
$$ \hat{\theta}_v \sim \mathcal{B}(\tilde{\alpha}_{\mathcal{T}}(v) \cdot \tau(p), \tilde{\beta}_{\mathcal{T}}(v) \cdot \tau(p)). $$
Scaling the Beta parameters by $\tau(p) \ge 1$ keeps the expected mean unchanged but strictly shrinks the variance ($\text{Var}[\theta] \propto 1/\tau(p)$). This systematically drives the policy from exploratory sampling to decisive exploitation as resource exhaustion approaches ($p \to 1$). The winner of this stochastic competition determines the subsequent search node:
\begin{equation}v_{next} = \operatorname*{argmax}{v \in \text{children}} \ \hat{\theta}_v.\end{equation}
This formulation ensures that the search explores diverse possibilities in the early stages (high variance) and asymptotically converges to the most promising solution (zero variance) as $p \to 1$.

This hierarchical formulation natively establishes a spatial equilibrium: shallow root nodes accumulate massive evidence from entire clades, driving their parameters toward deterministic exploitation. Conversely, deep frontier nodes aggregate evidence from sparse local subtrees, and the $\lambda$ factor isolates them from distant signals, preserving high variance to compel the exploration of novel semantic mutations~\cite{tesauro2012bayesian}. Empirical depth statistics (see Table~\ref{tab:depth_dynamics}) confirm that our mechanism maintains an adaptive preference structure even as the search tree grows deeper.

\subsection{Dynamic Clade Freezing}
To maximize the utility of the computational budget, we implement a \textbf{Dynamic Clade Freezing} mechanism. This strategy proactively prunes the search tree by identifying and halting the expansion of sub-optimal clades, thereby concentrating resources on regions with higher discovery potential. Formally, a clade rooted at $v$ is frozen if its aggregated performance demonstrates a significant gap relative to the global optimum:\begin{equation}\text{freeze}(v) \iff \text{visits}(v) \geq V_{\min} \,\ and \,\ \  \mu_{\mathcal{T}}(v) < \gamma \cdot \mu_{\text{global}}^{\star},\end{equation}where $\mu_{\mathcal{T}}(v)$ denotes the expected performance of the clade (computed as the simple average of all descendant normalized scores), and $\mu_{\text{global}}^{\star}$ represents the global maximum normalized score identified by the best heuristic found so far. $\gamma \in (0,1)$ serves as a strictness threshold.

To systematically prevent the premature freezing of high-variance, high-potential clades due to stochastic evaluation noise, this mechanism enforces two structural safeguards:
(1) \textbf{Statistical Grounding:} The criterion is exclusively evaluated if the node has accumulated sufficient historical visits ($\text{visits}(v) \geq V_{\min}$).
(2) \textbf{Clade-level Aggregation:} By comparing the global best against the aggregated $\mu_{\mathcal{T}}(v)$ rather than an isolated node score, the policy prevents the pruning of a promising lineage merely due to a few transiently poor, exploratory mutations. This acts as a robust filter, maximizing budget utility without sacrificing evolutionary diversity.

\subsection{Overall Workflow of the Proposed Algorithm}

Algorithm~\ref{alg:clade_ahd} formalizes the complete execution pipeline of Clade-AHD. The search process initializes the root node using a seed heuristic configured with uniform Bayesian priors. In each iteration—until the evaluation budget is exhausted—the framework systematically orchestrates five synergistic phases to navigate the vast heuristic space. \textbf{Phase 1 (Selection)} traverses the search tree strictly from the root by capitalizing on clade-level belief aggregation and budget-aware annealing. This mechanism naturally balances the exploration of novel regions with the exploitation of historically promising subtrees via guided Thompson Sampling. \textbf{Phase 2 (LLM-driven Expansion)} extends the selected leaf node by prompting a Large Language Model with customized evolutionary operators (e.g., mutation or crossover) to formulate a novel heuristic candidate. The generated candidate is subsequently assessed on the target dataset during \textbf{Phase 3 (Evaluation)}, where its empirical performance is normalized into fractional pseudo-counts to facilitate rigorous Bayesian inference. Next, in \textbf{Phase 4 (Bottom-up Belief Update)}, the evaluation outcome is monotonically backpropagated along the ancestral sequence, symmetrically updating the Beta distribution parameters of all encompassing parent nodes to refine the hierarchical belief model continuously. Finally, \textbf{Phase 5 (Dynamic Clade Freezing)} periodically evaluates the statistical viability of extensively visited clades. Subtrees whose aggregated empirical means fall critically below the performance threshold of the incumbent best heuristic are permanently pruned (frozen), thereby progressively constricting the search space and concentrating computational resources exclusively on high-potential heuristic regions. Upon termination, the framework definitively yields the globally optimal heuristic discovered throughout the search trajectory.

\section{Discussion}
\begin{figure*}[t]
    \centering
    \includegraphics[width=0.95\textwidth]{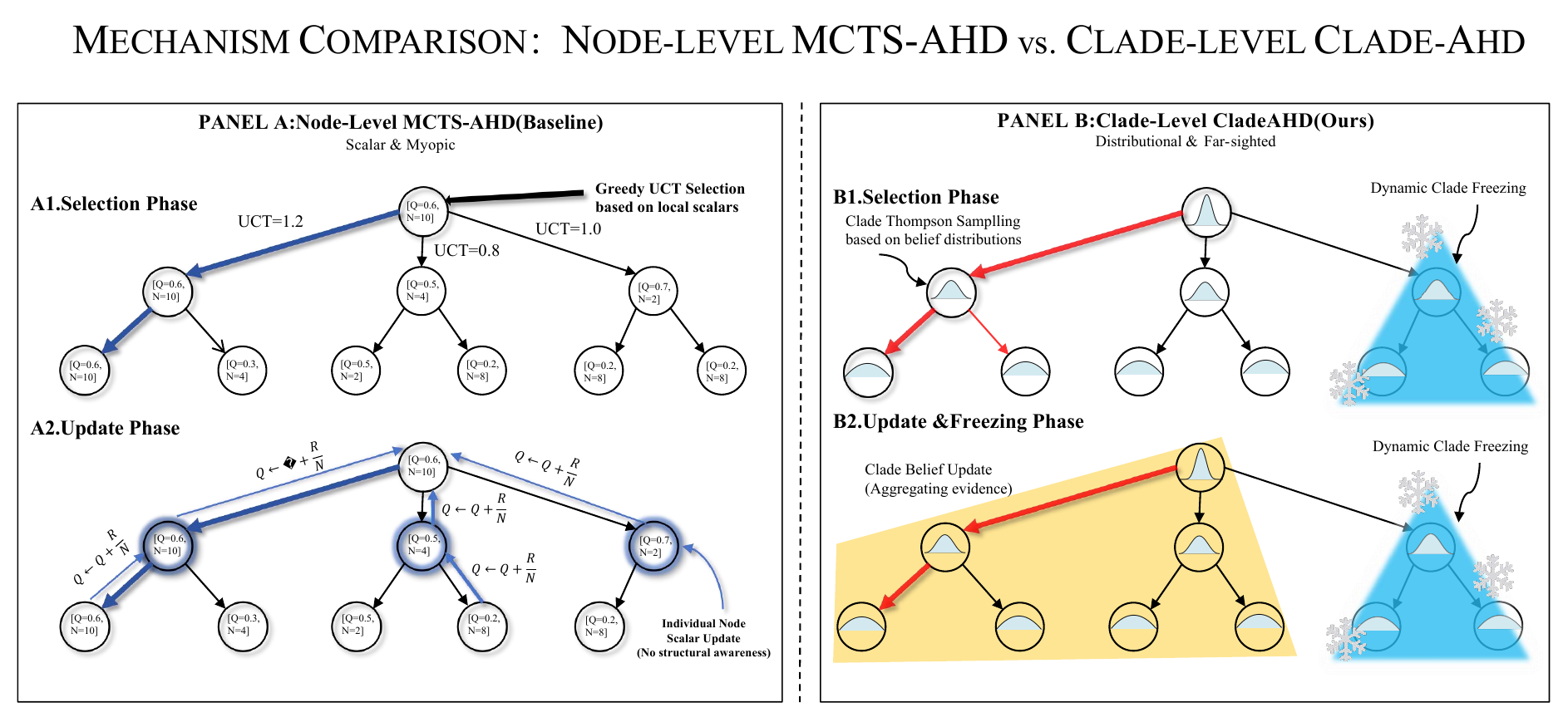}
    
    \caption{\textbf{Mechanism Comparison between Node-Level MCTS (MCTS-AHD) and Clade-Level MCTS (Clade-AHD).} 
    \textbf{(Panel A)} The baseline MCTS-AHD relies on scalar point estimates (e.g., mean value), leading to \textit{Scalar Myopia} where high-potential but high-variance branches are discarded due to initial noise. 
    \textbf{(Panel B)} Our Clade-AHD models potential as belief distributions. By aggregating evidence from the entire clade and applying \textit{Thompson Sampling}, it effectively identifies promising regions and reduces uncertainty (visualized as sharper posterior curves), while the \textit{Dynamic Freezing} mechanism prunes suboptimal branches.}
    \label{fig:mechanism_comparison}
\end{figure*}
In this section, we discuss the similarities and differences between MCTS-AHD~\cite{Zheng2025} and Clade-AHD. All of these algorithms leverage LLMs for AHD but employ different search strategies and efficiency optimization mechanisms.
\subsection{Comparison with MCTS-AHD}
\paragraph{Similarities between MCTS-AHD and Clade-AHD:}
Both MCTS-AHD and Clade-AHD employ MCTS as the core search framework to organize and explore the heuristic space. Similarly, both methods maintain a tree structure that preserves all historical information, enabling comprehensive exploration and avoiding premature convergence to local optima.

\paragraph{Differences between MCTS-AHD and Clade-AHD:}
A primary distinction lies in the selection mechanism. Where MCTS-AHD relies on the traditional UCT algorithm to evaluate individual nodes, Clade-AHD adopts a Clade-level Thompson sampling approach. This shift in perspective allows for a more strategic evaluation based on the collective evidence from all descendants in a subtree.
The two frameworks also diverge in their approach to credit assignment. Clade-AHD introduces a depth-attenuated backpropagation mechanism to address the temporal credit assignment problem, ensuring a node's influence diminishes with ancestral distance. This contrasts with the uniform credit distribution used in MCTS-AHD.

\subsection{Regret Analysis}
We provide theoretical analysis for two structural mechanisms underpinning Clade-AHD.

\begin{proposition}[Effective Influence under Depth-Attenuated Aggregation]\label{prop:influence_bound}
Let the search tree have expected branching factor $b$ and let the depth-attenuation factor be $\lambda \in (0,1)$. For any ancestor node $v$, the total expected influence of all descendants on its clade belief, measured in effective sample-size units, satisfies:
\begin{equation}
n_{\textup{eff}}(v) = \sum_{d=0}^{\infty} \mathbb{E}[N_d]\,\lambda^d \le \sum_{d=0}^{\infty} b^d \lambda^d = \frac{1}{1 - b\lambda},
\end{equation}
provided $b\lambda < 1$, i.e., $\lambda < 1/b$.
\end{proposition}
\begin{proof}
At depth offset $d$ from $v$, the expected number of descendants is at most $b^d$, each contributing evidence weighted by $\lambda^d$. The expected total weighted influence is bounded by the geometric series $\sum_{d=0}^{\infty} (b\lambda)^d = 1/(1-b\lambda)$, which converges iff $b\lambda < 1$.
\end{proof}

Proposition~\ref{prop:influence_bound} establishes a geometric convergence property: under a properly calibrated $\lambda$, the clade belief at any node is dominated by structurally proximate descendants. Distant generations contribute vanishing weight exponentially, preventing deep semantic drift from unboundedly distorting ancestral value estimates. This bound is derived under the worst-case assumption that all possible descendant positions are populated; MCTS selection bias concentrates expansion in higher-reward regions, further limiting the practical influence of deep branches beyond this conservative bound.

\begin{proposition}[False-Freezing Probability under Depth-Weighted Evidence]\label{prop:freezing_regret}
Let $v^*$ be the optimal clade with expected potential $\mu^*$. Its depth-attenuated clade mean is $\hat{\mu}_w = \sum_{i=1}^{N} \tilde{w}_i s_i$, where $\tilde{w}_i = \lambda^{d_i} / \sum_{j=1}^{N} \lambda^{d_j}$ are normalized depth weights, $s_i \in [0,1]$ are normalized scores, and $N \ge V_{\min}$. Under vanishing weight exponentially independent of the clade, the probability that the weighted mean falls below the freezing threshold satisfies:
\begin{equation}
\begin{aligned}
\mathbb{P}\bigl(\hat{\mu}_w < \gamma \mu^*\bigr) &\le \exp\!\Bigl(-2 N\,\eta\,(1 - \gamma)^2 (\mu^*)^2\Bigr) \\
&\le \exp\!\Bigl(-2 V_{\min}\,\eta\,(1 - \gamma)^2 (\mu^*)^2\Bigr),
\end{aligned}
\end{equation}
where $\eta = 1 / (N \sum_{i=1}^{N} \tilde{w}_i^2) \in (0, 1]$ is the weight-efficiency factor.
\end{proposition}
\begin{proof}
For independent $s_i \in [0,1]$ with normalized weights $\tilde{w}_i$, the weighted Hoeffding inequality~\cite{hoeffding1963probability} gives $\mathbb{P}(\hat{\mu}_w - \mathbb{E}[\hat{\mu}_w] \le -t) \le \exp(-2t^2 / \sum_i \tilde{w}_i^2)$. Setting $t = \mathbb{E}[\hat{\mu}_w] - \gamma\mu^*$ and assuming $\mathbb{E}[\hat{\mu}_w] = \mu^*$, i.e., that the depth-weighting is unbiased for the clade's expected potential, yields $t = (1-\gamma)\mu^*$. Substituting $\sum_i \tilde{w}_i^2 = 1/(N\eta)$ gives the exponent $-2N\eta t^2$. Since the freezing rule only triggers when $N \ge V_{\min}$ and the right-hand side is decreasing in $N$, the probability is upper-bounded by evaluating at the minimum required sample size $N = V_{\min}$. When all evaluations reside at equal depth, $\tilde{w}_i = 1/N$, $\sum_i \tilde{w}_i^2 = 1/N$, and $\eta = 1$, recovering the standard unweighted Hoeffding bound. For non-uniform depths, $\tilde{w}_i$ concentrates mass on shallower nodes, yielding $\sum_i \tilde{w}_i^2 > 1/N$ and thus $\eta < 1$, which directly quantifies the effective sample-size reduction due to depth-weighting.
\end{proof}

Proposition~\ref{prop:freezing_regret} establishes that the false-freezing probability retains exponential decay in $V_{\min}$, with $\eta$ capturing the cost of non-uniform depth weighting. Several idealizations qualify this result. The global-best reference $\mu_{\textup{global}}^*$ appearing in the actual freezing rule is itself a random variable subject to upward noise; replacing it with the constant $\mu^*$ omits a source of false-positive variance, and the bound should therefore be interpreted as an idealized baseline rather than a tight guarantee. The assumption of conditional independence fails to account for the inherent structural correlations among LLM-generated heuristics derived from a shared ancestor, thereby degrading the effective sample size more severely than the $\eta$ adjustment implies. Moreover, $\mathbb{E}[\hat{\mu}_w] = \mu^*$ holds only when depth is uncorrelated with heuristic quality; systematic score drift across generations would introduce an additional bias term not captured here. Despite these idealizations, the bound conveys a robust qualitative message: increasing $V_{\min}$ and maintaining balanced tree depth jointly suppress erroneous pruning. Practical calibration remains informed by the empirical sensitivity analyses in Table~\ref{tab:sensitivity_gamma} and Table~\ref{tab:sensitivity_vmin}.

\begin{algorithm}[h!]
\caption{Clade-AHD: Clade-level Bayesian MCTS for Heuristic Design}
\label{alg:clade_ahd}
\small 
\begin{algorithmic}[1]
\REQUIRE Evaluation Budget $T_{max}$, Initial Population $N_{init}$, Dataset $D$
\REQUIRE Bayesian Priors $\alpha_0, \beta_0$, Decay Factor $\lambda$, Temperature Param $\omega_{cool}$
\REQUIRE Pruning Threshold $\gamma$, Min Visits $V_{min}$
\ENSURE Best Heuristic $h^*$

\STATE \textbf{Initialization:}
\STATE $v_{root} \gets \textsc{CreateNode}(h_{seed}, \alpha_0, \beta_0)$
\STATE $\mathcal{T} \gets \{v_{root}\}$, $evals \gets 0$, $h^* \gets h_{seed}$

\WHILE{$evals < T_{max}$}
    \STATE \color{blue}// Phase 1: Selection via Clade-Level Thompson Sampling \color{black}
    \STATE $v_{next} \gets v_{root}$
    \WHILE{$v_{next}$ is not Leaf}
        \STATE Candidates $\mathcal{C} \gets \{u \in v_{next}.children \mid u.frozen = \text{False}\}$
        \IF{$\mathcal{C}$ is empty} 
            \STATE \textbf{break} 
        \ENDIF        
        \FOR{each child $u \in \mathcal{C}$}
            \STATE \textbf{// 1. Clade Belief Aggregation}
            \STATE $\alpha_{\mathcal{T}}(u) \gets 1 + \sum_{d \in \text{desc}(u)} \lambda^{\text{dist}(u, d)} (\alpha_d - 1)$
            \STATE $\beta_{\mathcal{T}}(u) \gets 1 + \sum_{d \in \text{desc}(u)} \lambda^{\text{dist}(u, d)} (\beta_d - 1)$
            
            \STATE \textbf{// 2. Budget-Aware Annealing }
            \STATE $p \gets evals / T_{max}$
            \STATE $\tau \gets (1/(1-p))^{\omega_{cool}}$
            \STATE $\tilde{\alpha} \gets \alpha_{\mathcal{T}}(u) \cdot \tau, \quad \tilde{\beta} \gets \beta_{\mathcal{T}}(u) \cdot \tau$
            
            \STATE \textbf{// 3. Thompson Sampling}
            \STATE $\theta_u \sim \text{Beta}(\tilde{\alpha}, \tilde{\beta})$
        \ENDFOR
        \STATE $v_{next} \gets \text{argmax}_{u \in \mathcal{C}} (\theta_u)$
    \ENDWHILE

    \STATE \color{blue}// Phase 2: Expansion (LLM Generation) \color{black}
    \STATE Select operator $op \in \{\text{Mutation}, \text{Crossover}\}$
    \STATE $h_{new} \gets \text{LLM}(v_{next}.context, op)$
    \STATE $v_{new} \gets \textsc{CreateNode}(h_{new}, \alpha_0, \beta_0)$
    \STATE $v_{next}.children.\text{add}(v_{new})$, $\mathcal{T}.\text{add}(v_{new})$

    \STATE \color{blue}// Phase 3: Evaluation \color{black}
    \STATE $score \gets \text{Evaluate}(h_{new}, D)$
    \STATE Update global best $h^*$ if $score > \text{score}(h^*)$
    \STATE $outcome \gets \textsc{Normalize}(score)$ \COMMENT{Convert to fractional pseudo-counts}
    \STATE $evals \gets evals + 1$

    \STATE \color{blue}// Phase 4: Bottom-Up Belief Update \color{black}
    \STATE $v \gets v_{new}$
    \WHILE{$v \neq \text{NULL}$}
        \STATE $\alpha_v \gets \alpha_v + outcome$
        \STATE $\beta_v \gets \beta_v + (1 - outcome)$
        \STATE $v \gets v.parent$
    \ENDWHILE

    \STATE \color{blue}// Phase 5: Dynamic Clade Freezing \color{black}
    \IF{$evals \pmod{N_{check}} == 0$}
        \FOR{$v \in \mathcal{T}$ where $v.visits > V_{min}$}
             \STATE $\mu_{\mathcal{T}} \gets \alpha_{\mathcal{T}}(v) / (\alpha_{\mathcal{T}}(v) + \beta_{\mathcal{T}}(v))$
             \IF{$\mu_{\mathcal{T}} < \gamma \cdot \text{score}(h^*)$}
                 \STATE $v.frozen \gets \text{True}$ \COMMENT{Prune entire clade}
             \ENDIF
        \ENDFOR
    \ENDIF
\ENDWHILE
\STATE \textbf{return} $h^*$
\end{algorithmic}
\end{algorithm}

\section{Experiments}
In this section, we comprehensively evaluate Clade-AHD. We first detail the experimental settings and baselines. Then, we assess its performance on standard benchmarks (Sec. VII-B) and its zero-shot generalization to large-scale instances (Sec. VII-C). Finally, we conduct an in-depth analysis of its search dynamics (Sec. VII-D) and perform extensive ablation studies to validate its core components (Sec. VII-E).

\subsection{Experimental Setup}
In this section, we evaluate the proposed \textbf{Clade-AHD} on a diverse set of CO problems. To demonstrate versatility, we apply \textbf{Clade-AHD} to design core heuristic functions for two distinct algorithmic paradigms: \textbf{constructive heuristics} and \textbf{Ant Colony Optimization (ACO)}.

\textbf{Settings.}
We initialize the search with $N_I=4$ nodes and maintain a population size of 10. For Bayesian inference, we use priors $\alpha=\beta=1.0$ with 10 pseudo-evaluations for stability. The annealing parameter is set to $\omega_{cool}=1.0$, and the backpropagation decay factor is $\lambda=0.8$ (determined to be optimal through a sensitivity analysis comparing $\lambda \in \{0.0, 0.5, 0.8\}$). We apply dynamic pruning with a threshold of 0.1 and a minimum visit count of 10. To handle score variance, we use adaptive normalization at the 90th percentile and enable parallel operators for efficiency. During evaluation, each heuristic is restricted to a 180 seconds timeout on the dataset $\mathcal{D}$. Detailed configurations of the dataset and framework are provided below. To verify the generalization of our framework, we evaluate \textbf{Clade-AHD} using \texttt{GPT-4o-mini} and \texttt{GLM-4-flash}.

\textbf{Baselines.}
To validate the effectiveness of heuristics generated by \textbf{Clade-AHD}, we compare our approach against four categories of benchmarks: 
\textbf{Manual Heuristics:} Classic baselines including Nearest-greedy~\cite{romera2024mathematical}, ACO~\cite{dorigo2007ant}, and Expected Improvement (EI)~\cite{movckus1974bayesian}.
\textbf{Traditional AHD:} The Genetic Programming-based method GHPP~\cite{duflo2019gp}.
\textbf{Neural Combinatorial Optimization (NCO):} Representative learning-based solvers under similar frameworks, specifically POMO~\cite{kwon2020pomo} and DeepACO~\cite{ye2023deepaco}, with recent advances extending to colored TSP variants via attention-based RL~\cite{zhu2024reinforcement} and multi-objective TSP via meta-DRL~\cite{fu2025towards}.
\textbf{LLM-based AHD:} State-of-the-art frameworks including FunSearch~\cite{romera2024mathematical}, EoH~\cite{liu2024evolution}, ReEvo~\cite{ye2024reevo}, and the recent HSEvo~\cite{dat2025hsevo}.

Regarding initialization, FunSearch, ReEvo, and HSEvo require a handcrafted seed function; to ensure fairness, we provide the same rudimentary seed for these methods without additional external knowledge. In contrast, EoH and our \textbf{Clade-AHD} initiate evolution without manual seeds, demonstrating superior autonomy and applicability.

All experiments are conducted on a single AMD-8845HS CPU. Following the protocol in \cite{liu2024evolution}, we set the evaluation budget for all LLM-based AHD methods to $T=1,000$ queries on the evaluation dataset $\mathcal{D}$.

\subsection{Performance on Standard Benchmarks}
To demonstrate the versatility of Clade-AHD, we first evaluate its performance across six standard NP-hard CO benchmarks using two distinct algorithmic paradigms: step-by-step construction and Ant Colony Optimization (ACO). By testing on both paradigms, we verify that the generated heuristics are framework-agnostic and robust to different search methodologies.

\subsubsection{Results under Constructive Framework}
The constructive framework (or step-by-step construction) builds a feasible solution node-by-node~\cite{asani2023computation}. Beyond its role in LLM-based AHD, it serves as a foundational paradigm for Neural Combinatorial Optimization (NCO) methods~\cite{vinyals2015pointer,bello2016neural}. We utilize this framework to design heuristics for TSP, KP, and Online BPP. For all LLM-based AHD methods, the evaluation dataset $\mathcal{D}$ consists of 64 TSP instances ($N=50$) and 64 KP instances ($N=100$, Capacity $W=25$).

As shown in Table~\ref{tab:step-by-step} and Table~\ref{tab:step-bpp-online}, Clade-AHD consistently outperforms existing LLM-based AHD methods and manual heuristics on TSP, KP, and Online BPP. Notably, when tested on larger instances ($N=200$), \textbf{Clade-AHD} surpasses the advanced NCO method POMO. Crucially, unlike POMO, which mandates task-specific training, Clade-AHD operates without pre-training, demonstrating superior zero-shot generalization on NP-hard problems.

\begin{table*}[t]
\centering
\caption{Designing heuristics with the step-by-step construction framework for TSP and KP. We evaluate methods on 6 test sets with 1,000 instances each. Test sets with in-domain scales (i.i.d. to the evaluation dataset D) are underlined. Since AHD methods have no guarantees for generalization ability, the effect on in-domain datasets is more important. Optimal for TSP is obtained by LKH (Lin \& Kernighan, 1973), and Optimal for KP is the result of OR-Tools. Each LLM-based AHD method is run three times, and we report the average performance. The best-performing method with each LLM is shaded, and each test set's overall best result is in bold.}
\label{tab:step-by-step}
\renewcommand{\arraystretch}{1.1}
\setlength{\tabcolsep}{3pt}
\footnotesize
\begin{tabular}{@{}lcccccccccccc@{}}
\toprule
\multicolumn{1}{c}{Task} & \multicolumn{6}{c}{\textbf{TSP}} & \multicolumn{6}{c}{\textbf{KP}} \\
\cmidrule(lr){2-7} \cmidrule(l){8-13}
\multicolumn{1}{c}{N=} & \multicolumn{2}{c}{\underline{$N$=50}} & \multicolumn{2}{c}{\underline{$N$=100}} & \multicolumn{2}{c}{$N$=200} & \multicolumn{2}{c}{\underline{$N$=50, $W$=12.5}} & \multicolumn{2}{c}{\underline{$N$=100, $W$=25}} & \multicolumn{2}{c}{$N$=200, $W$=25} \\
\cmidrule(lr){2-3} \cmidrule(lr){4-5} \cmidrule(lr){6-7} \cmidrule(lr){8-9} \cmidrule(lr){10-11} \cmidrule(l){12-13}
\multicolumn{1}{c}{Methods} & Obj.$\downarrow$ & Gap & Obj.$\downarrow$ & Gap & Obj.$\downarrow$ & Gap & Obj.$\uparrow$ & Gap & Obj.$\uparrow$ & Gap & Obj.$\uparrow$ & Gap \\ \midrule
Optimal & 5.675 & - & 7.768 & - & 10.659 & - & 20.037 & - & 40.271 & 0.75\% & 57.448 & 0.68\% \\
Greedy Construct & 6.959 & 22.62\% & 9.706 & 24.94\% & 13.461 & 26.29\% & 19.985 & 0.26\% & 40.225 & 0.86\% & 57.395 & 0.77\% \\
POMO & \textbf{5.697} & \textbf{0.39\%} & \textbf{8.001} & \textbf{3.01\%} & 12.897 & 20.45\% & 19.612 & 2.12\% & 39.676 & 2.22\% & 57.271 & 0.98\% \\ \midrule
\multicolumn{13}{c}{\textit{LLM-based AHD: GPT-4o-mini}} \\ \midrule
Funsearch & 6.357 & 12.00\% & 8.850 & 13.93\% & 12.372 & 15.54\% & 19.988 & 0.24\% & 40.227 & 0.86\% & 57.398 & 0.76\% \\
EoH & 6.394 & 12.67\% & 8.894 & 14.49\% & 12.437 & 16.68\% & 19.993 & 0.22\% & 40.231 & 0.85\% & 57.399 & 0.76\% \\
MCTS-AHD & \textbf{6.225} & \textbf{9.69\%} & 8.684 & 11.79\% & 12.120 & 13.71\% & 20.015 & 0.11\% & 40.252 & 0.80\% & 57.423 & 0.72\% \\
\rowcolor{gray!20} Clade-AHD & 6.265 & 10.39\% & \textbf{8.484} & \textbf{9.22\%} & \textbf{11.864} & \textbf{11.30\%} & 20.013 & 0.12\% & \textbf{40.576} & \textbf{0.00\%} & \textbf{57.840} & \textbf{0.00\%} \\ \bottomrule
\end{tabular}
\end{table*}

\begin{table}[h!]
\caption{Design step-by-step construction heuristics for online BPP. The table exhibits the performance gaps of heuristics to the lower bound. Each LLM-based AHD method is run three times for the average gaps. In-domain scales are underlined.}
\label{tab:step-bpp-online}
\centering
\renewcommand{\arraystretch}{1.25}
\setlength{\tabcolsep}{3pt}
\resizebox{\columnwidth}{!}{
\begin{tabular}{@{}lccccccc@{}}
\toprule
\multicolumn{1}{c}{Test sets} & 1k\_100 & 1k\_500 & 5k\_100 & 5k\_500 & 10k\_100 & 10k\_500 & Avg. \\ 
\midrule
Best Fit & 4.77\% & 0.25\% & 4.31\% & 0.55\% & 4.05\% & 0.47\% & 2.40\% \\
First Fit & 5.02\% & 0.25\% & 4.65\% & 0.55\% & 4.36\% & 0.50\% & 2.56\% \\ \midrule
\multicolumn{8}{c}{LLM-based AHD: GPT-4o-mini} \\ \midrule
Funsearch & 2.45\% & 0.66\% & 1.30\% & \textbf{0.25\%} & 1.05\% & \textbf{0.21\%} & 0.99\% \\
EoH & 2.69\% & 0.25\% & 1.63\% & 0.53\% & 1.47\% & 0.45\% & 1.17\% \\
ReEvo & 3.94\% & 0.50\% & 2.72\% & 0.40\% & 2.39\% & 0.31\% & 1.71\% \\
HSEvo & 2.64\% & 1.07\% & 1.43\% & 0.32\% & 1.13\% & \textbf{0.21\%} & 1.13\% \\
MCTS-AHD & 2.45\% & 0.50\% & \textbf{1.06\%} & 0.32\% & \textbf{0.74\%} & 0.26\% & 0.89\% \\ 
\rowcolor{gray!20}
Clade-AHD & \textbf{2.32\%} & \textbf{0.00\%} & 1.22\% & 0.30\% & 1.07\% & 0.27\% & \textbf{0.86\%} \\ \bottomrule
\end{tabular}%
}
\end{table}

\subsubsection{Results under ACO Framework}
ACO is a meta-heuristic inspired by the foraging behavior of ants. It utilizes a heuristic matrix and simulates pheromone-based communication to guide probabilistic path selection for solving CO problems~\cite{dorigo2007ant,kim2024ant}. In the context of AHD, LLMs can be leveraged to automatically design the \textit{generation function} for the heuristic matrix, thereby generalizing the ACO framework to diverse tasks. 

Furthermore, when integrated into the ACO framework (Table~\ref{tab:aco}), the heuristics evolved by Clade-AHD achieve state-of-the-art results, surpassing both traditional ACO baselines and advanced NCO solvers like DeepACO on in-domain scales. \textbf{Clade-AHD} demonstrates significant improvements over existing LLM-based methods (EoH and ReEvo) across all in-domain test sets and three out-of-domain scenarios.

\begin{table*}[t]
\centering
\caption{Designing heuristics with the ACO general framework for solving TSP, CVRP, MKP, and offline BPP. Each test set contains 64 instances, and the LLM-based AHD methods' performances are averaged over three runs.}
\label{tab:aco}
\renewcommand{\arraystretch}{1.15}
\setlength{\tabcolsep}{1.5pt}
\scriptsize
\resizebox{\textwidth}{!}{%
    \begin{tabular}{@{}l cccccccccccccccc@{}}
    \toprule
    \multirow{3}{*}{\textbf{Method}} & \multicolumn{4}{c}{\textbf{TSP}} & \multicolumn{4}{c}{\textbf{CVRP}} & \multicolumn{4}{c}{\textbf{MKP}} & \multicolumn{4}{c}{\textbf{Offline BPP}} \\
    \cmidrule(lr){2-5} \cmidrule(lr){6-9} \cmidrule(lr){10-13} \cmidrule(l){14-17}
    & \multicolumn{2}{c}{$N$=50} & \multicolumn{2}{c}{$N$=100} & \multicolumn{2}{c}{$N$=50, $C$=50} & \multicolumn{2}{c}{$N$=100, $C$=50} & \multicolumn{2}{c}{$N$=200, $m$=5} & \multicolumn{2}{c}{$N$=200, $m$=5} & \multicolumn{2}{c}{$N$=500} & \multicolumn{2}{c}{$N$=1k} \\
    \cmidrule(lr){2-3} \cmidrule(lr){4-5} \cmidrule(lr){6-7} \cmidrule(lr){8-9} \cmidrule(lr){10-11} \cmidrule(lr){12-13} \cmidrule(lr){14-15} \cmidrule(l){16-17}
    & Obj.$\downarrow$ & Gap & Obj.$\downarrow$ & Gap & Obj.$\downarrow$ & Gap & Obj.$\downarrow$ & Gap & Obj.$\uparrow$ & Gap & Obj.$\uparrow$ & Gap & Obj.$\downarrow$ & Gap & Obj.$\downarrow$ & Gap \\
    \midrule
    ACO & 5.992 & 4.66\% & 8.948 & 9.70\% & 11.355 & 27.77\% & 18.778 & 25.76\% & 22.738 & 2.28\% & 40.672 & 4.30\% & 208.828 & 2.81\% & 417.938 & 3.15\% \\
    DeepACO & 5.842 & 2.04\% & 8.282 & 1.53\% & \textbf{8.888} & \textbf{0.00\%} & \textbf{14.932} & \textbf{0.00\%} & 23.093 & 0.75\% & 41.988 & 1.20\% & \textbf{203.125} & \textbf{0.00\%} & \textbf{405.172} & \textbf{0.00\%} \\
    \midrule
    \multicolumn{17}{c}{\textit{LLM-based AHD: GPT-4o-mini}} \\
    \midrule
    EoH & 5.828 & 1.80\% & 8.263 & 1.30\% & 9.359 & 5.31\% & 15.681 & 5.02\% & 23.139 & 0.56\% & 41.994 & 1.19\% & 204.646 & 0.75\% & 408.599 & 0.85\% \\
    ReEvo & 5.856 & 2.29\% & 8.340 & 2.24\% & 9.327 & 4.94\% & 16.092 & 7.77\% & 23.245 & 0.10\% & 42.416 & 0.19\% & 206.693 & 1.76\% & 413.510 & 2.06\% \\
    MCTS-AHD & 5.801 & 1.33\% & 8.179 & 0.27\% & 9.286 & 4.48\% & 15.782 & 5.70\% & \textbf{23.269} & \textbf{0.00\%} & \textbf{42.498} & \textbf{0.00\%} & 204.094 & 0.48\% & 407.323 & 0.53\% \\
    \rowcolor{gray!20} Clade-AHD & \textbf{5.725} & \textbf{0.00\%} & \textbf{8.157} & \textbf{0.00\%} & \textbf{9.269} & \textbf{4.48\%} & \textbf{15.544} & \textbf{0.77\%} & 23.167 & 0.32\% & 42.152 & 0.39\% & \textbf{203.365} & \textbf{0.00\%} & \textbf{405.245} & \textbf{0.01\%} \\
    \bottomrule
    \end{tabular}%
}
\end{table*}

\subsection{Zero-Shot Generalization on Large-Scale Instances}
A critical limitation of current LLM-based AHD methods is their tendency to overfit small-scale, synthetic distributions (e.g., $N=50$ or $100$). To rigorously assess the zero-shot generalization capability of Clade-AHD, we scale our evaluation to complex, real-world benchmarks, specifically evaluating on TSPLib and CVRPLib datasets without any fine-tuning.

\subsubsection{Generalization on TSPLib}
While evaluating on 64 stochastic instances has become a standard protocol in contemporary AHD literature~\cite{liu2024evolution}, such restricted scales often fail to expose the inherent robustness of optimization algorithms. They provide only a limited glimpse into performance. To bridge this evaluation gap, we rigorously tested Clade-AHD on large-scale TSPLib benchmarks featuring over 1,000 nodes. 

Table~\ref{tab:comparison} details the performance on TSPLib instances featuring over 1,000 nodes. Clade-AHD delivers superior solution quality across all highly-dimensional spatial distributions, confirming that our clade-level abstraction prevents structural brittleness.

\begin{table}[htbp]
\centering
\small
\setlength{\tabcolsep}{4pt}
\caption{Average solution costs on large-scale TSPLib instances (over 1,000 nodes) over three independent runs.}
\label{tab:comparison}
\begin{tabular}{lrrrr}
\toprule
Instance & EoH & FunSearch & ReEvo & Clade-AHD (Ours) \\
\midrule
pr1002   & 328998 & 328998 & 328998 & \textbf{309822} \\
dsj1000  & 24631468 & 24171092 & 24631468 & \textbf{21770798} \\
u1060    & 303035 & 298491 & 303035 & \textbf{268710} \\
vm1084   & 300550 & 300130 & 300550 & \textbf{292902} \\
rl1304   & 320653 & 322147 & 320653 & \textbf{302571} \\
rl1323   & 325747 & 330227 & 325747 & \textbf{322558} \\
nrw1379  & 69498 & 69686 & 69498 & \textbf{63883} \\
u1432    & 183647 & 183647 & 183647 & \textbf{179695} \\
d1655    & 74535 & 74535 & 74535 & \textbf{72035} \\
rl1889   & 400620 & 390396 & 400620 & \textbf{377880} \\
vm1748   & 421325 & 419512 & 421325 & \textbf{400142} \\
pcb1173  & 72092 & 72447 & 72092 & \textbf{68657} \\
\bottomrule
\end{tabular}
\end{table}

\subsubsection{Generalization on CVRPLib}
This robustness extends to routing problems with complex topological constraints. To rigorously evaluate the proposed framework on routing problems exhibiting diverse topological characteristics, we expanded our analysis to the widely recognized CVRPLib benchmark. 

As Table~\ref{tab:group_comparison} demonstrates, Clade-AHD uniformly minimizes average costs across 182 heterogeneous CVRPLib instances (Groups A through X), substantially outperforming EoH, FunSearch, and ReEvo. Notably, whether navigating tightly clustered customer nodes or randomly dispersed points under strict capacity constraints, our method maintains robust superiority. This uniform success across heterogeneous datasets substantiates that the clade-level exploration mechanism effectively mitigates topological brittleness.

\begin{table}[htbp]
\centering
\small
\setlength{\tabcolsep}{2pt}
\caption{Comparison of average solution costs for CVRPLib instance groups. Each entry reports the mean cost over all instances in the group, comparing EoH, FunSearch, ReEvo, and Clade-AHD.}
\label{tab:group_comparison}
\begin{tabular}{lrrrrr}
\toprule
Group & Instances & EoH & FunSearch & ReEvo & Clade-AHD \\
\midrule
A & 27 & 1736.04 & 1741.58 & 1898.38 & \textbf{1393.47} \\
B & 18 & 1635.44 & 1495.43 & 1516.39 & \textbf{1280.62} \\
E & 11 & 1132.23 & 1168.41 & 1375.87 & \textbf{973.19} \\
F & 3  & 1475.72 & 1234.09 & 1388.32 & \textbf{1000.98} \\
M & 4  & 2068.31 & 1916.33 & 2146.78 & \textbf{1698.17} \\
P & 21 & 893.08  & 898.58  & 1024.37 & \textbf{750.46} \\
X & 98 & 92322.31 & 147415.49 & 122397.07 & \textbf{74433.65} \\
TOTAL & 182 & 50372.56 & 80020.59 & 66609.13 & \textbf{40612.27} \\
\bottomrule
\end{tabular}
\end{table}

\subsection{Search Efficiency and Convergence Dynamics}
Beyond final solution quality, we investigate the search dynamics of Clade-AHD to understand its advantage in navigating the heuristic space $\mathcal{H}$.

To validate the effectiveness of our approach, we illustrate the performance evolution curves, using \texttt{GLM-4-flash}, for designing step-by-step construction heuristics for the TSP. Each curve represents the average performance over three independent runs, with shaded regions indicating the standard deviation.

Figure~\ref{fig:convergence_curves} illustrates the performance evolution curves on the TSP. Unlike population-based methods (e.g., EoH) and scalar-based MCTS-AHD, which exhibit steep initial gains but plateau prematurely due to early exploitation, Clade-AHD maintains a sustained optimization trajectory. This empirically validates that our Clade-level Thompson Sampling effectively bypasses suboptimal traps during the early budget phase, unlocking a higher performance ceiling. Distinguishing itself from population-based evolutionary computation (EC) methods, Clade-AHD demonstrates superior sample efficiency and exploration capabilities. This enables a more comprehensive exploration of the heuristic space $\mathcal{H}$, effectively preventing the search from becoming trapped in local optima.

\begin{figure}[t]
    \centering
    \includegraphics[width=\linewidth]{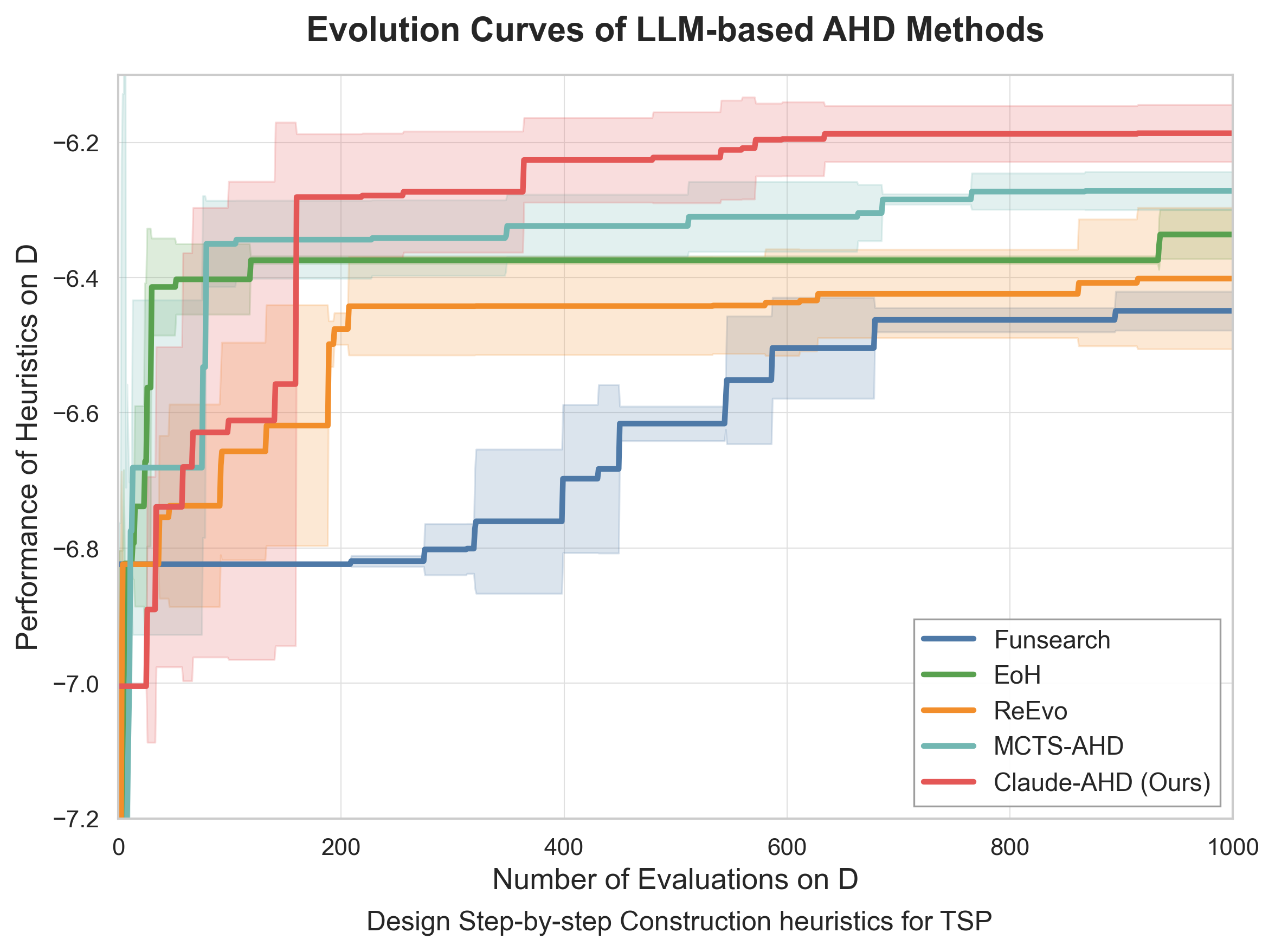}
    \caption{Evolution curves of LLM-based AHD methods on designing Step-by-step Construction heuristics for TSP. The solid lines represent the mean performance over three independent runs, while the shaded regions indicate the standard deviation. Clade-AHD demonstrates superior asymptotic performance and robust search capabilities compared to population-based baselines and MCTS-AHD.}
    \label{fig:convergence_curves}
\end{figure}

\subsection{Ablation and Parameter Sensitivity Analysis}
To isolate the contributions of individual algorithmic designs and assess the framework's robustness, we conduct extensive ablation and sensitivity studies.

\subsubsection{Component Ablations}
To validate the contribution of individual components within \textbf{Clade-AHD}, we conducted ablation studies on the KP and TSP. We constructed five variants, each omitting a specific module, and evaluated them over three independent runs. 

Table~\ref{tab:ablation_components} presents the performance gap upon removing key components. \textit{Dynamic Clade Freezing} and \textit{Temperature Annealing} emerge as the most critical modules, highlighting the necessity of proactively pruning redundant computations and modulating the exploration-exploitation trade-off. Furthermore, \textit{Depth-Attenuated Backpropagation} is essential for preserving long-term dependencies against signal decay in deep trees. Finally, \textit{Pseudo-Evaluation} and \textit{Adaptive Normalization} enforce stability, preventing metric inflation and ensuring consistent estimation despite their subtler magnitude.

\begin{table}[h!]
\centering
\caption{Ablation study on the components of Clade-AHD.We use different variants to design heuristics on TSP50 and KP100 benchmarks with the step-by-step frameworks over three turns.}
\label{tab:ablation_components}
\renewcommand{\arraystretch}{1.25}
\setlength{\tabcolsep}{6pt}
\resizebox{\columnwidth}{!}{
\begin{tabular}{@{}lcc@{}}
\toprule
\textbf{Configuration} & \textbf{TSP50} & \textbf{KP100} \\
\midrule
Clade-AHD (Full Model) & 6.269 (0.000\%) & 40.576 (0.000\%) \\
w/o Pseudo-Evaluations & +1.869\% & +0.862\% \\
w/o Depth-Attenuated Credit & +4.950\% & +0.143\% \\
w/o Dynamic Clade Freezing & +7.751\% & +0.148\% \\
w/o Temperature Annealing & +6.657\% & +0.416\% \\
w/o Adaptive Normalization & +1.870\% & +0.064\% \\
\bottomrule
\end{tabular}%
}
\end{table}

\subsubsection{Deep Dive into Exploration Mechanisms}
We further analyze the structural flaws of frequentist estimation in sparse environments. As shown in Table~\ref{tab:ablation_exploration}, indiscriminately scaling the UCT exploration constant $C$ fails to improve performance. Clade-AHD circumvents this by enforcing explicitly epistemic uncertainty. By modeling clade-level potential via $\mathcal{B}(\alpha_v, \beta_v)$ distributions, the framework quantifies the probability of optimality with much higher precision. Furthermore, the integration of expectation-based pseudo-counts ensures bounded initial variance. This stabilization mechanism prevents noisy, single-evaluation outliers from inducing disproportionate exploration spikes, thereby maintaining a robust and informed convergence trajectory.

\begin{table}[htbp]
\centering
\small
\setlength{\tabcolsep}{2pt} 
\caption{Ablation study on exploration parameters. We compare standard UCT under various exploration constants ($C \in \{0.05, 0.2, 0.5\}$) against MCTS-AHD and Clade-AHD on TSP and KP benchmarks.}
\label{tab:ablation_exploration}
\begin{tabular}{lcccc}
\toprule
Model & Parameter & TSP50 & KP100 & TSP100 \\
\midrule
Standard UCT & $C=0.05$ & \textbf{+11.08\%} & \textbf{+0.056\%} & -- \\
Standard UCT & $C=0.2$  & \textbf{+12.12\%} & \textbf{+0.034\%} & -- \\
Standard UCT & $C=0.5$  & \textbf{+10.11\%} & \textbf{+0.042\%} & -- \\
Baseline MCTS-AHD & $\lambda=0.1$ & Base & Base & 11.79\% \\
Clade-AHD & (Optimized) & -- & \textbf{0.00\%} & \textbf{9.22\%} \\
\bottomrule
\end{tabular}
\end{table}

\begin{table}[htbp]
\centering
\caption{Global depth and layer-wise exploration dynamics in the search tree (mean over 5 independent runs). We trace the maximum tree depth, the proportion of visits to deep nodes ($d \ge 5$), and the corresponding entropy of node visits at the Root versus Depth-4 level across the evaluation budget (10\% to 100\%).}
\label{tab:depth_dynamics}
\renewcommand{\arraystretch}{1.1}
\setlength{\tabcolsep}{3pt}
\resizebox{\columnwidth}{!}{%
\begin{tabular}{@{}lcccc@{}}
\toprule
Budget & Max Depth & Share ($d \ge 5$) & Root Ent. & Depth-4 Ent. \\
\midrule
10\%  & 4.8  & 15.8\% & 0.696 & 0.963 \\
50\%  & 5.8  & 23.8\% & 0.612 & 0.912 \\
90\%  & 6.8  & 29.2\% & 0.585 & 0.880 \\
100\% & 9.8  & 33.0\% & 0.572 & 0.875 \\
\bottomrule
\end{tabular}%
}
\end{table}

Additionally, Tables~\ref{tab:sensitivity_gamma} and~\ref{tab:sensitivity_vmin} validate that our dynamic freezing mechanisms remain robust under optimal thresholding ($\gamma = 0.1$) and sufficient statistical grounding ($V_{\min} = 10$). When the threshold is set too high (e.g., $0.3$), the selection policy becomes excessively restrictive, leading to the premature elimination of high-potential lineages. Conversely, a threshold that is too low (e.g., $0.05$) or entirely absent ($\gamma = 0$) fails to effectively compress the redundant search space.

\begin{table}[htbp]
\centering
\small
\setlength{\tabcolsep}{5pt}
\caption{Sensitivity analysis of the Dynamic Freezing threshold $\gamma$ on TSP-50 and TSP-100.}
\label{tab:sensitivity_gamma}
\begin{tabular}{lccr}
\toprule
Parameter & Task & Tour Length & Change \\
\midrule
$\gamma=0.1$ & TSP-50 / 100 & \textbf{6.562} / \textbf{8.961} & -- (Base) \\
$\gamma=0.3$ & TSP-50 / 100 & 7.348 / 10.276 & \textbf{+11.98\% / +14.67\%} \\
$\gamma=0.05$ & TSP-50 / 100 & 7.355 / 10.014 & \textbf{+12.08\% / +11.75\%} \\
$\gamma=0$ & TSP-50 / 100 & 7.071 / -- & \textbf{+7.75\% / --} \\
\bottomrule
\end{tabular}
\end{table}

\begin{table}[htbp]
\centering
\small
\setlength{\tabcolsep}{5pt}
\caption{Sensitivity analysis of the minimum visits requirement $V_{\min}$ .}
\label{tab:sensitivity_vmin}
\begin{tabular}{lcr}
\toprule
$V_{\min}$ & Task & Performance Degradation \\
\midrule
10 (Base) & TSP-50 / KP-100 & -- \\
0 & TSP-50 / KP-100 & \textbf{+7.75\% / +0.15\%} \\
\bottomrule
\end{tabular}
\end{table}

\subsubsection{Impact of Temporal Credit Assignment}

Finally, Figure~\ref{fig:lambda_sensitivity} visualizes the impact of the depth-attenuated decay factor $\lambda$. The myopic setting ($\lambda = 0$) suffers from severe premature convergence, substantiating our core hypothesis: aggregating hierarchical clade evidence ($\lambda = 0.8$) is indispensable for mitigating evaluation sparsity in AHD.

\begin{figure*}[h]
    \centering
    \includegraphics[width=0.95\textwidth]{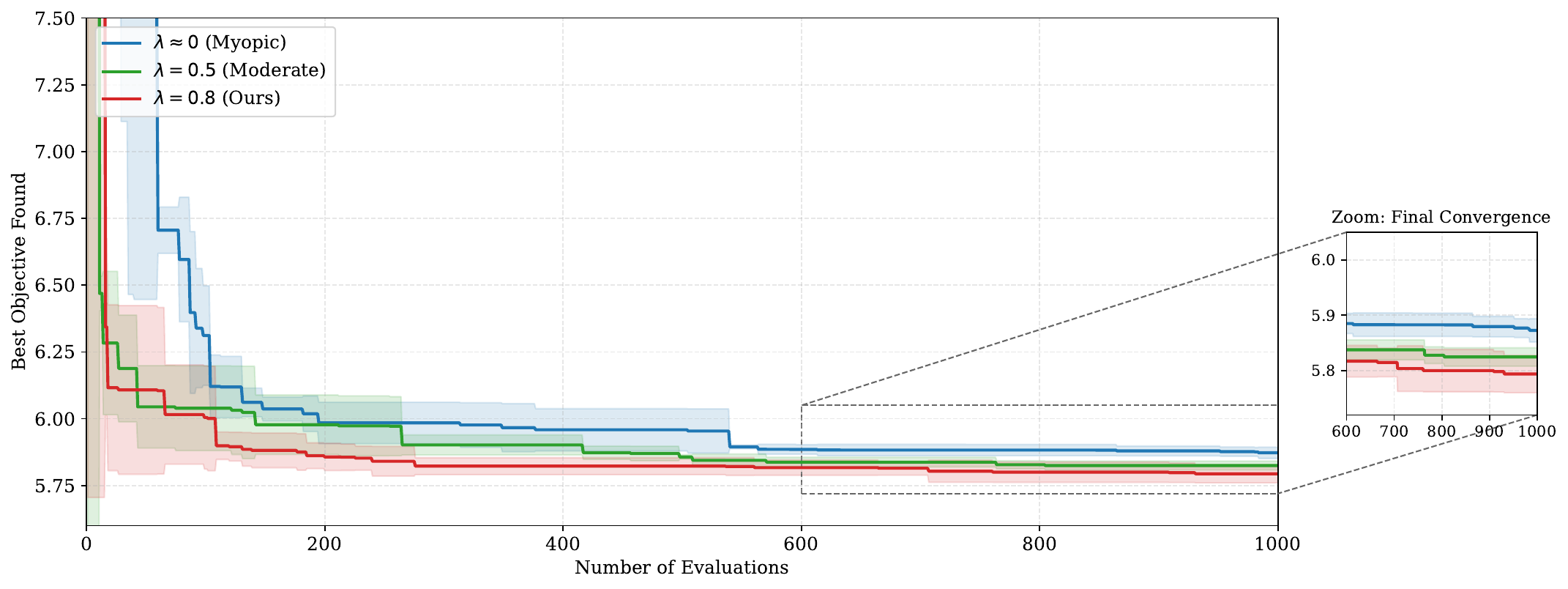}
    \vspace{-0.2cm}
    \caption{\textbf{Sensitivity Analysis of the Decay Factor $\lambda$.} 
    The main plot (left) shows the best objective value found over 1,000 evaluations. The \textbf{blue line} ($\lambda = 0$) represents a myopic, node-centric baseline, which suffers from premature convergence. The \textbf{red line} ($\lambda = 0.8$, Ours) demonstrates the most robust convergence trajectory. The \textbf{zoom panel} (right) highlights the final optimization stage (iterations 600-1000), showing that $\lambda=0.8$ consistently achieves the lowest objective value with narrower variance compared to $\lambda=0.5$ and $\lambda = 0$.}
    \label{fig:lambda_sensitivity}
\end{figure*}

\section{Conclusion}

While Monte Carlo Tree Search (MCTS) shows promise in Automatic Heuristic Design (AHD), its traditional reliance on node-level point estimates severely limits sample efficiency under strictly bounded evaluation budgets. To overcome this, we introduce Clade-AHD, a novel framework that reformulates the search paradigm from isolated nodes to evolutionary clades. We first transition from noisy node-level scalar estimates to clade-level Bayesian beliefs, utilizing Beta distributions to create a distributional abstraction of the search space. We then construct a depth-attenuated belief update mechanism that aggregates descendant evaluations, effectively mitigating semantic drift and evaluation sparsity. To navigate this probabilistic space, we employ a Clade-level Thompson Sampling strategy coupled with a dynamic clade freezing mechanism, which explicitly models epistemic uncertainty to balance exploration and exploitation under strict computational budgets.

Objective evaluation is addressed by comprehensive experiments on standard NP-hard combinatorial optimization problems, including TSP, KP, and CVRP. 
Key findings include: a) Clade-AHD models consistently outperform state-of-the-art LLM-based AHD baselines in both solution quality and sample effectiveness; b) through the clade-level abstraction, Clade-AHD demonstrates robust zero-shot generalization over unseen, large-scale realistic instances (e.g., TSPLib and CVRPLib) without falling into local optima; c) ablation studies underscore that each tailored design in Clade-AHD consistently contributes to the final model's search efficiency; and d) Clade-AHD effectively circumvents the structural flaws of frequentist estimation in sparse environments, identifying high-potential evolutionary clades that standard methods often overlook.

As a pioneering work in Bayesian MCTS for Automatic Heuristic Design, Clade-AHD holds certain limitations. On the one hand, the current performance exhibits sensitivity to prompt design, requiring empirical calibration to align with specific LLM reasoning capabilities. An immediate and promising direction for future work is to integrate automated prompt optimization mechanisms to eliminate this dependency. On the other hand, the current framework primarily focuses on single-objective scenarios. We envision extending Clade-AHD to accommodate more complex, out-of-distribution problem types, such as multi-objective optimization scenarios.

In summary, beyond the immediate technical gains, we hope this research will inspire the community to further explore the potential of hierarchical Bayesian paradigms in automated algorithm design, ultimately advancing practical resource allocation in broader domains such as logistics, manufacturing, and chip design.


\bibliographystyle{IEEEtran}
\bibliography{references}



\end{document}


\title{Supplementary Material to Clade-AHD: Clade-level Selection for Efficient MCTS in Automatic Heuristic Design}
\author{}

\maketitle

\section{Definition of Tasks}
\label{sec:task_definitions}
In this section, we provide formal mathematical formulations for the optimization tasks employed to evaluate Clade-AHD. These tasks encompass six representative NP-hard CO problems\cite{Zheng2025}.

\subsection{NP-hard CO Problems}

Our evaluation benchmark includes the Traveling Salesman Problem (TSP), Capacitated Vehicle Routing Problem (CVRP), 0-1 Knapsack Problem (KP), Multiple Knapsack Problem (MKP), Admissible Set Problem (ASP), and Bin Packing Problem (BPP).

\textbf{Traveling Salesman Problem (TSP).}
The TSP serves as a cornerstone in combinatorial optimization theory \cite{Garey1979}. The objective is to determine the minimal-cost Hamiltonian cycle that visits a set of $N$ cities exactly once and returns to the origin. Given a distance matrix $D = \{d_{ij}\}$ where $d_{ij}$ denotes the Euclidean distance between node $i$ and node $j$, the problem is formulated as minimizing the total tour length:
\begin{equation}
    \min_{\pi} \sum_{i=1}^{N-1} d_{\pi(i), \pi(i+1)} + d_{\pi(N), \pi(1)}
\end{equation}
where $\pi$ represents a permutation of the city indices $\{1, \dots, N\}$. We evaluate our method on standard Euclidean TSP instances.

\textbf{Capacitated Vehicle Routing Problem (CVRP).}
Generalizing the TSP, the CVRP involves optimizing routes for a fleet of vehicles to serve a set of customers subject to capacity constraints \cite{Toth2014}. This paper follows the settings of ReEvo\cite{ye2024reevo} when generating CVRP data sets, fixing the depot coordinates to (0.5, 0.5).  Each instance consists of a depot (node 0) and $N$ customer nodes. The objective is to minimize the total travel distance of all vehicles such that: (1) each customer is visited exactly once; (2) all routes start and end at the depot; and (3) the total demand on any route does not exceed the vehicle capacity $Q$.
Mathematically, let $s = \{\rho^1, \dots, \rho^k\}$ be a solution with $k$ routes. The objective function is defined as:
\begin{equation}
    f(s) = \sum_{j=1}^{k} \left( d_{0, \rho^j_1} + \sum_{t=1}^{|\rho^j|-1} d_{\rho^j_t, \rho^j_{t+1}} + d_{\rho^j_{|\rho^j|}, 0} \right)
\end{equation}
subject to the capacity constraint $\sum_{node \in \rho^j} \text{d}_{node} \le Q$ for all routes $j$.

\textbf{0-1 Knapsack Problem (KP).}
The KP represents a fundamental resource allocation problem. Given a knapsack with capacity $W$ and a set of $N$ items, each characterized by a value $v_i$ and a weight $w_i$, the goal is to select a subset of items to maximize the total value without violating the capacity constraint.
\begin{equation}
    \max_{x} \sum_{i=1}^{N} v_i x_i \quad \text{s.t.} \quad \sum_{i=1}^{N} w_i x_i \le W, \quad x_i \in \{0, 1\}
\end{equation}
Following established protocols \cite{kwon2020pomo}, we generate instances where weights and values are sampled from uniform distributions.

\textbf{Multiple Knapsack Problem (MKP).}
The MKP extends the standard KP by introducing multiple resource constraints. Instead of a single container, we have $m$ distinct knapsacks with capacities $C_1, \dots, C_m$.  We still follow the setting of Reevo\cite{ye2024reevo} for the MKP instances, and the objective remains to maximize the total value of selected items, ensuring that the subset of items assigned to each knapsack $j$ does not exceed its specific capacity $C_j$. This problem introduces higher complexity due to the coupled constraints across multiple resources.

\textbf{Bin Packing Problem (BPP).}
The BPP aims to partition a set of items with varying sizes into the minimum number of bins, each with a fixed capacity $C$.  \textbf{Online BPP:} Items arrive sequentially. The algorithm must strictly decide which bin to place the current item in before observing future items. This requires robust heuristic logic to handle uncertainty. We utilize Weibull distribution-based instances for evaluation, as they provide a challenging benchmark for bin packing heuristics \cite{castineiras2012weibull}\cite{liu2024evolution}.

\section{Definition of General Frameworks}
\label{sec:frameworks_definition}

To verify the framework-agnosticism of Clade-AHD, we apply it to design key heuristic functions within two distinct general frameworks: Step-by-Step Construction and Ant Colony Optimization (ACO). This section details the operational logic of these frameworks.

\subsection{Step-by-Step Construction}
The Step-by-Step Construction framework (also known as the Constructive Heuristic framework) is an intuitive approach capable of handling a wide range of CO problems. It constructs a complete solution from scratch by sequentially selecting decision variables. In each step, the framework evaluates all valid candidates using a specific priority function, and the candidate with the highest priority is added to the partial solution.

In this work, Clade-AHD is employed to automatically design this \textbf{key priority function}. We apply this framework to four tasks:
\begin{itemize}
    \item \textbf{TSP:} The heuristic function selects the next city to visit based on the current city, the set of unvisited cities, and the distance matrix. It effectively learns a "next-node-prediction" policy similar to Pointer Networks \cite{vinyals2015pointer}.
    \item \textbf{KP:} The function selects the next item to pack into the knapsack. It inputs the value and weight of remaining items and the current residual capacity, outputting a priority score for each item. This generalizes the classical greedy strategy based on the value-to-weight ratio.
    \item \textbf{Online BPP:} For each incoming item, the heuristic assigns a preference score to each available bin (and the option of opening a new bin) based on the item size and the bins' residual capacities.
\end{itemize}

\subsection{Ant Colony Optimization (ACO)}

Ant Colony Optimization (ACO) is a meta-heuristic inspired by the foraging behavior of ants \cite{dorigo2007ant}. It constructs solutions probabilistically based on a pheromone matrix $\tau$ (historical experience) and a heuristic matrix $\eta$ (problem-specific desirability).

Following the established protocol in \textbf{ReEvo \cite{ye2024reevo}}, we utilize Clade-AHD to automatically evolve the \textit{Heuristic Matrix Generation Function}. Instead of using a fixed formula (e.g., $\eta_{ij} = 1/d_{ij}$), the LLM generates a Python function that computes $\eta$ dynamically based on the specific constraints of each problem. The specific formulations for the involved tasks are detailed below:

\begin{itemize}
    \item \textbf{Traveling Salesman Problem (TSP):}
    The problem is modeled as a graph traversal. The LLM generates a function to compute an $N \times N$ matrix $\eta$, where $\eta_{ij}$ represents the desirability of moving from city $i$ to city $j$. The generated heuristic typically balances Euclidean distance with other geometric features to guide the ant's path construction.

    \item \textbf{Capacitated Vehicle Routing Problem (CVRP):}
    Similar to TSP, this is an edge-based selection problem but with capacity constraints. The heuristic function inputs the distance matrix, node demands, and vehicle capacity. It outputs a matrix $\eta_{ij}$ that guides the vehicle to the next customer $j$ from $i$. A high-quality heuristic here must learn to balance minimizing travel distance with optimizing the remaining vehicle load (e.g., avoiding visiting a distant customer that consumes the last bit of capacity).

    \item \textbf{Multiple Knapsack Problem (MKP):}
    Unlike the routing problems, MKP is a subset selection problem. The constructive process involves selecting items to place into knapsacks. The LLM designs a function that calculates a heuristic vector (or matrix) $\eta$, representing the priority of selecting item $j$ given the current resource consumption of the knapsacks. The heuristic effectively learns a generalized "value-per-weight" density estimation across multiple dimensions.

    \item \textbf{Offline Bin Packing Problem (BPP):}
    The goal is to pack items into the minimum number of bins. In the ACO framework, this is treated as selecting the best "next item" to pack into the current bin. The LLM generates a function that computes the desirability $\eta_j$ for each remaining item $j$, considering its size and the residual capacity of the current bin. This allows the algorithm to evolve complex strategies beyond standard "Best Fit" or "First Fit" rules.
\end{itemize}

By adopting this unified interface, Clade-AHD leverages the global search capability of the underlying ACO solver while replacing the handcrafted, problem-specific heuristic rules with data-driven logic evolved by our clade-based search.

\section{Experimental Setup and Implementation Details}
\label{appendix_D}

To ensure the transparency, reproducibility, and fairness of our comparative study, this appendix provides a comprehensive description of the experimental protocols. We detail the evaluation logic, dataset composition, computational environment, and specific hyperparameter configurations for both Clade-AHD and the baseline methods.

\subsection{Evaluation Protocols}
The setup of evaluation budgets ($T$) and evaluation datasets ($D$) in this work generally follows the standard protocols established by FunSearch \cite{romera2024mathematical}, EoH \cite{liu2024evolution}, and ReEvo\cite{ye2024reevo}. Adopting these community-recognized settings ensures that our performance gains are attributable to the search strategy itself rather than skewed experimental conditions.

\paragraph{The Setting of Evaluation Budget $T$.}
In population-based methods like EoH \cite{liu2024evolution}, the total number of evaluations is typically determined by $\text{generations} \times \text{population\_size}$ (e.g., $20 \times 20 \approx 400 \sim 1,600$ queries).
To align with this computational magnitude while ensuring sufficient exploration for our Bayesian beliefs, we standardize the maximum evaluation budget to $T = 1,000$ for all tasks (including TSP, KP, and Online BPP). This strict budget constraint challenges the algorithm to be sample-efficient, distinguishing efficient search strategies from brute-force enumeration.

\paragraph{The Setting of Evaluation Dataset $D$.}
The composition of $D$ is critical for preventing overfitting. Our design logic balances fairness with robustness:
\begin{itemize}
    \item \textbf{Standardization for Fairness (General Tasks):} For TSP, CVRP, KP, and MKP, we strictly adopt the standard synthetic instances used in MCTS-AHD \cite{Zheng2025} and EoH \cite{liu2024evolution}. This ensures a direct "apples-to-apples" comparison with baselines.
    \item \textbf{Variation for Robustness (Special Case: Online BPP):} Standard baselines often use a fixed dataset (e.g., 5,000 items with $W=100$). However, heuristics evolved on such uniform data frequently fail to generalize to other scales (e.g., $W=500$). To address this, we implement a \textbf{varying-scale setup} similar to \cite{castineiras2012weibull}. Our evaluation dataset $D$ comprises a mixture of Weibull distribution instances with diverse characteristics (varying item counts and capacities), compelling the heuristic to learn generalized packing logic.
\end{itemize}

The detailed specifications of $D$ are summarized in Table \ref{tab:eval_datasets_detailed}.

\begin{table*}[t]
    \centering
    \caption{Detailed Specifications of Evaluation Datasets ($\mathbf{D_S}$).}
    \label{tab:eval_datasets_detailed}
    \small
    \renewcommand{\arraystretch}{1.3}
    \begin{tabularx}{\textwidth}{ll >{\raggedright\arraybackslash}X}
        \toprule
        \textbf{Task} & \textbf{Dataset Size} & \textbf{Instance Details (Attributes)} \\
        \midrule
        \textbf{TSP} & 64 Instances & $N=50$ nodes, coordinates $s \sim \mathcal{U}[0,1]^{2d}$. \\
        \textbf{CVRP} & 64 Instances & $N=50$ nodes, Demand $s \sim \mathcal{U}[1,9]$, Capacity varies. \\
        \textbf{KP} & 64 Instances & $N=100$ items, Capacity $S_W=25$, uncorrelated weights/values. \\
        \textbf{MKP} & 64 Instances & $N=50$ items, $M=5$ knapsacks. \\
        \textbf{Online BPP.} & \textbf{Mixture (4 Sets)} & \textbf{Varying Scales (Weibull Dist.):} A robust mix of Short ($N \approx 1k$) vs. Long ($N \approx 5k$) sequences, and Small vs. Large Bin Capacities. \\
        \bottomrule
    \end{tabularx}
\end{table*}

\subsection{Computational Environment}
\textbf{Hardware Setup.}
Unlike many Neural Combinatorial Optimization (NCO) approaches that rely on heavy GPU acceleration (e.g., POMO \cite{kwon2020pomo}, DeepACO \cite{ye2023deepaco}), Clade-AHD is designed to be computationally lightweight. All experiments were conducted on a personal workstation:
\begin{itemize}
    \item \textbf{CPU:} AMD Ryzen 7 8845HS (8 cores, 16 threads).
    \item \textbf{GPU:} None (Inference is purely CPU-based).
    \item \textbf{Software:} Python 3.10, leveraging \texttt{NumPy} for vectorization.
\end{itemize}

\textbf{LLM Configuration.}
We utilized \textbf{GPT-4o-mini} (\texttt{gpt-4o-mini-2024-07-18}) as the core reasoning engine. The temperature is set to $T=1.0$ during the \textit{Expansion} phase to maximize code diversity, and $T=0.2$ during \textit{Selection} logic (if aided) for stability. The context window includes the task description, parent code, and a history of the top-3 past successful mutations.

\subsection{Baseline and Hyperparameter Configuration}
To validate the effectiveness of Clade-AHD, we compare it against four state-of-the-art LLM-based AHD methods under identical budget constraints ($T=1,000$):

Table \ref{tab:hyperparameters} summarizes the specific hyperparameters used in Clade-AHD, empirically tuned on a separate validation set.

\begin{table}[h]
    \centering
    \caption{Hyperparameter settings for Clade-AHD.}
    \label{tab:hyperparameters}
    \footnotesize
    \setlength{\tabcolsep}{3pt}
    \begin{tabular}{l|c|p{3.5cm}}
    \toprule
    \textbf{Parameter} & \textbf{Value} & \textbf{Description} \\
    \midrule
    $N_{init}$ & 4 & Number of initial random heuristics \\
    $K$ (Population) & 10 & Maximum active population size \\
    \midrule
    \textit{Bayesian Inference} & & \\
    $\alpha_{prior}, \beta_{prior}$ & 1.0 & Uniform Beta priors (uninformative) \\
    $n_{pseudo}$ & 10 & Pseudo-counts for prior stabilization \\
    $\omega_{cool}$ & 1.0 & Annealing rate for Thompson Sampling \\
    \midrule
    \textit{Tree Update} & & \\
    $\lambda$ & 0.8 & Decay factor for Depth-Attenuated Backpropagation \\
    $V_{min}$ & 10 & Minimum visits required before freezing \\
    $\gamma$ & 0.1 & Pruning threshold for Dynamic Clade Freezing \\
    \midrule
    \textit{Evaluation} & & \\
    Timeout & 180s & Max execution time per heuristic on dataset $D$ \\
    \bottomrule
    \end{tabular}
\end{table}

\section{Detailed Methodology}
\label{sec:appendix_e}

\subsection{Prompts of MCTS Actions}

To ensure a strictly fair comparison and to isolate the contribution of our Clade-based search strategy from prompt engineering, \textbf{we adopt the exact prompt templates provided in MCTS-AHD \cite{Zheng2025} without any modification}.

We utilize the six distinct actions ($i1, e1, e2, m1, m2, s1$) defined in MCTS-AHD for initialization and expansion. The specific prompts for these actions are presented below. In all examples, we use the Traveling Salesman Problem (TSP) with the step-by-step construction framework as the representative task context.

\begin{promptbox}{Prompt for Action i1 (Initialization)}
\small
\textbf{Task Context:} \\
Solving Traveling Salesman Problem (TSP) with constructive heuristics. TSP requires finding the shortest path that visits all given nodes and returns to the starting node.

\textbf{Instruction:} \\
First, describe the design idea and main steps of your algorithm in one sentence. The description must be inside a brace outside the code implementation. Next, implement it in Python as a function named 'select\_next\_node'. \\
This function should accept 4 input(s): 'current\_node', 'destination\_node', 'unvisited\_nodes', 'distance\_matrix'. \\
The function should return 1 output(s): 'next\_node'. The select next node function takes as input the current node, the destination node, a set of unvisited nodes, and a distance matrix, and returns the next node to visit. \\
Do not give additional explanations.
\end{promptbox}

\begin{promptbox}{Prompt for Action m1 (Mutation \textminus{} Mechanism)}
\small
\textbf{Task Context:} \\
Solving Traveling Salesman Problem (TSP) with constructive heuristics. TSP requires finding the shortest path that visits all given nodes and returns to the starting node.

\textbf{Input Context:} \\
I have one algorithm with its code as follows: \\
\# Its Description \\
\# It's Python Code Implementation of A Function

\textbf{Instruction:} \\
Please create a new algorithm that has a different form, but can be a modified version of the provided algorithm. Attempt to introduce more novel mechanisms and new equations or programme segments. \\
First, describe the design idea and main steps of your algorithm in one sentence. The description must be inside a brace outside the code implementation. Next, implement it in Python as a function named 'select\_next\_node'. \\
This function should accept 4 input(s): 'current\_node', 'destination\_node', 'unvisited\_nodes', 'distance\_matrix'. \\
The function should return 1 output(s): 'next\_node'. The select next node function takes as input the current node, the destination node, a set of unvisited nodes, and a distance matrix, and returns the next node to visit. \\
Do not give additional explanations.
\end{promptbox}

\begin{promptbox}{Prompt for Action m2 (Mutation \textminus{} Parameter)}
\small
\textbf{Task Context:} \\
Solving Traveling Salesman Problem (TSP) with constructive heuristics. TSP requires finding the shortest path that visits all given nodes and returns to the starting node.

\textbf{Input Context:} \\
I have one algorithm with its code as follows: \\
\# Its Description \\
\# It's Python Code Implementation of A Function

\textbf{Instruction:} \\
Please identify the main algorithm parameters and help me in creating a new algorithm that has different parameter settings for equations compared to the provided algorithm. \\
First, describe the design idea and main steps of your algorithm in one sentence. The description must be inside a brace outside the code implementation. Next, implement it in Python as a function named 'select\_next\_node'. \\
This function should accept 4 input(s): 'current\_node', 'destination\_node', 'unvisited\_nodes', 'distance\_matrix'. \\
The function should return 1 output(s): 'next\_node'. The select next node function takes as input the current node, the destination node, a set of unvisited nodes, and a distance matrix, and returns the next node to visit. \\
Do not give additional explanations.
\end{promptbox}

\begin{promptbox}{Prompt for Action e1 (Crossover \textminus{} Divergent)}
\small
\textbf{Task Context:} \\
Solving Traveling Salesman Problem (TSP) with constructive heuristics. TSP requires finding the shortest path that visits all given nodes and returns to the starting node.

\textbf{Input Context:} \\
I have $k$ existing algorithms with their codes as follows: \\
No.1 algorithm's description, its corresponding code, and its objective value are: \\
... \\
No.$k$ algorithm's description, its corresponding code, and its objective value are: \\
...

\textbf{Instruction:} \\
Please create a new algorithm that has a totally different form from the given algorithms. Try generating codes with different structures, flows, or algorithms. The new algorithm should have a relatively low objective value. \\
First, describe the design idea and main steps of your algorithm in one sentence. The description must be inside a brace outside the code implementation. Next, implement it in Python as a function named 'select\_next\_node'. \\
This function should accept 4 input(s): 'current\_node', 'destination\_node', 'unvisited\_nodes', 'distance\_matrix'. \\
The function should return 1 output(s): 'next\_node'. The select next node function takes as input the current node, the destination node, a set of unvisited nodes, and a distance matrix, and returns the next node to visit. \\
Do not give additional explanations.
\end{promptbox}

\begin{promptbox}{Prompt for Action e2 (Crossover \textminus{} Elitist)}
\small
\textbf{Task Context:} \\
Solving Traveling Salesman Problem (TSP) with constructive heuristics. TSP requires finding the shortest path that visits all given nodes and returns to the starting node.

\textbf{Input Context:} \\
I have 2 existing algorithms with their codes as follows: \\
No.1 algorithm's description, its corresponding code, and its objective value are: \\
... \\
No.2 algorithm's description, its corresponding code,e and its objective value are: \\
...

\textbf{Instruction:} \\
Please create a new algorithm that has a similar form to the No.2 algorithm and is inspired by the No.1 algorithm. The new algorithm should have an objective value lower than both algorithms. \\
Firstly, list the common ideas in the No.1 algorithm that may give good performances. Secondly, based on the common idea, describe the design idea based on the No.len(indivs) algorithm and main steps of your algorithm in one sentence. The description must be inside a brace. Next, implement it in Python as a function named 'select\_next\_node'. \\
This function should accept 4 input(s): 'current\_node', 'destination\_node', 'unvisited\_nodes', 'distance\_matrix'. \\
The function should return 1 output(s): 'next\_node'. The select next node function takes as input the current node, the destination node, a set of unvisited nodes, and a distance matrix, and returns the next node to visit. \\
Do not give additional explanations.
\end{promptbox}

\begin{promptbox}{Prompt for Action s1 (Tree-Path Reasoning)}
\small
\textbf{Task Context:} \\
Solving Traveling Salesman Problem (TSP) with constructive heuristics. TSP requires finding the shortest path that visits all given nodes and returns to the starting node.

\textbf{Input Context:} \\
I have $k$ existing algorithms with their codes as follows: \\
No.1 algorithm's description, its corresponding code, and its objective value are: \\
... \\
No.$k$ algorithm's description, its corresponding code, and its objective value are: \\
...

\textbf{Instruction:} \\
Please help me create a new algorithm that is inspired by all the above algorithms with its objective value lower than any of them. \\
Firstly, list some ideas in the provided algorithms that are clearly helpful to a better algorithm. Secondly, based on the listed ideas, describe the design idea and main steps of your new algorithm in one sentence. The description must be inside a brace. Thirdly, implement it in Python as a function named 'select\_next\_node'. \\
This function should accept 4 input(s): 'current\_node', 'destination\_node', 'unvisited\_nodes', 'distance\_matrix'. \\
The function should return 1 output(s): 'next\_node'. The select next node function takes as input the current node, the destination node, a set of unvisited nodes, and a distance matrix, and returns the next node to visit. \\
Do not give additional explanations.
\end{promptbox}

\section{Licenses}
\label{sec:appendix_f}

To ensure transparency and facilitate reproducibility, we strictly adhere to the open-source licenses of all baselines and datasets used in this work. The specific licenses and access URLs are summarized in Table \ref{tab:licenses}.

\begin{table*}[t]
    \centering
    \caption{A summary of licenses for baselines and resources used in this work.}
    \label{tab:licenses}
    \begin{tabular}{l|c|l|l}
    \toprule
    \textbf{Resources} & \textbf{Type} & \textbf{License} & \textbf{URL} \\
    \midrule
    LKH3 & Code & Academic research use & \url{http://webhotel4.ruc.dk/~keld/research/LKH-3/} \\
    OR-Tools & Code & MIT License & \url{https://developers.google.com/optimization/} \\
    POMO & Code & Available online & \url{https://github.com/yd-kwon/POMO} \\
    DeepACO & Code & MIT License & \url{https://github.com/henry-yeh/DeepACO} \\
    VRP-DACT & Code & MIT License & \url{https://github.com/yining043/VRP-DACT} \\
    NeuOpt & Code & MIT License & \url{https://github.com/yining043/NeuOpt} \\
    FunSearch & Code & Apache License 2.0 & \url{https://github.com/google-deepmind/funsearch} \\
    EoH & Code & MIT License & \url{https://github.com/FeiLiu36/EoH} \\
    ReEvo & Code & MIT License & \url{https://github.com/ai4co/reevo} \\
    HSEvo & Code & Available online & \url{https://github.com/datphamvn/HSEvo} \\
    \midrule
     Synthetic Datasets & Dataset & Available online & \url{https://github.com/FeiLiu36/EoH/tree/main/examples} \\
    TSPLib & Dataset & Non-commercial use & \url{http://comopt.ifi.uni-heidelberg.de/software/TSPLIB95} \\
    \bottomrule
    \end{tabular}
\end{table*}

\bibliographystyle{IEEEtran}
\bibliography{references}